\documentclass{article}

\usepackage{microtype}
\usepackage{titletoc}
\usepackage{graphicx}
\usepackage{subfigure}
\usepackage{enumitem}
\usepackage{booktabs}
\usepackage{multirow}

\usepackage{amsmath}
\usepackage{cite}
\usepackage{amsthm}
\usepackage{amssymb}
\usepackage{algorithm}

\usepackage{colortbl}
\usepackage{xcolor}       
\usepackage{algorithmic}
\newtheorem{theorem}{Theorem}
\newtheorem{assumption}{Assumption}
\newtheorem{lemma}{Lemma}

\usepackage{hyperref}
\usepackage{makecell}

\usepackage[accepted]{icml2024}

\usepackage{amsmath}
\usepackage{amssymb}
\usepackage{mathtools}
\usepackage{amsthm}

\usepackage[capitalize,noabbrev]{cleveref}

\theoremstyle{plain}

\theoremstyle{definition}

\theoremstyle{remark}

\usepackage[textsize=tiny]{todonotes}

\begin{document}

\twocolumn[
\icmltitle{
Locally Estimated Global Perturbations are Better than Local Perturbations \\ for Federated Sharpness-aware Minimization
}

\begin{icmlauthorlist}
\icmlauthor{Ziqing Fan}{sjtu,pjlab}
\icmlauthor{Shengchao Hu}{sjtu,pjlab}
\icmlauthor{Jiangchao Yao}{sjtu,pjlab}
\icmlauthor{Gang Niu}{riken}
\icmlauthor{Ya Zhang}{sjtu,pjlab}
\icmlauthor{Masashi Sugiyama}{riken,tyo}
\icmlauthor{Yanfeng Wang}{sjtu,pjlab}
\end{icmlauthorlist}

\icmlaffiliation{sjtu}{Cooperative Medianet Innovation Center, Shanghai Jiao Tong University, China;}
\icmlaffiliation{pjlab}{Shanghai AI Laboratory, China;}
\icmlaffiliation{riken}{RIKEN AIP, Japan;}
\icmlaffiliation{tyo}{The University of Tokyo, Japan}

\icmlcorrespondingauthor{Jiangchao Yao and Yanfeng Wang}{\{sunarker,wangyanfeng\}@sjtu.edu.cn}

\icmlkeywords{Machine Learning, ICML}

\vskip 0.3in
]

\printAffiliationsAndNotice{}
\begin{abstract}
In federated learning (FL), the multi-step update
and data heterogeneity among clients often lead to a loss landscape with sharper minima, degenerating the performance of the resulted global model.
Prevalent federated approaches incorporate sharpness-aware minimization (SAM) into local training to mitigate this problem.
However, the local loss landscapes may not accurately reflect the flatness of global loss landscape in heterogeneous environments; as a result, minimizing local sharpness and calculating perturbations on client data might not align the efficacy of SAM in FL with centralized training.
To overcome this challenge, we propose FedLESAM, a novel algorithm that locally estimates the direction of global perturbation on client side as the difference between global models received in the previous active and current rounds.
Besides the improved quality, FedLESAM also speed up federated SAM-based approaches since it only performs once backpropagation in each iteration.
Theoretically, we prove a slightly tighter bound than its original FedSAM by ensuring consistent perturbation. 
Empirically, we conduct comprehensive experiments on four federated benchmark datasets under three partition strategies to demonstrate the superior performance and efficiency of FedLESAM\footnote{Our code is available at: \url{https://github.com/MediaBrain-SJTU/FedLESAM}}. 

\end{abstract}

\begin{figure}[t]
    \centering
    \subfigure[Centralized]{
    \centering
    \label{fig:intro_1}
    \includegraphics[width=0.145\textwidth]{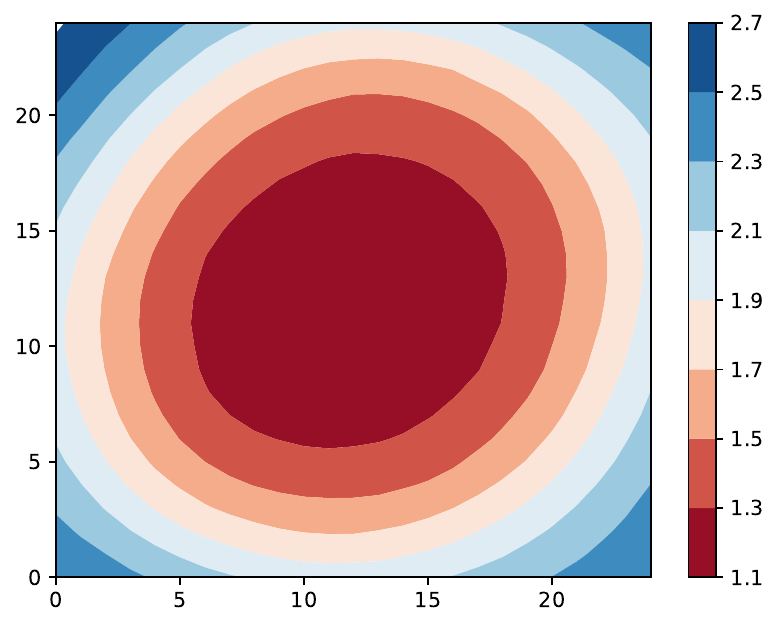}}
    \subfigure[Federated~($0.6$)]{
    \centering
    \label{fig:intro_2}
    \includegraphics[width=0.145\textwidth]{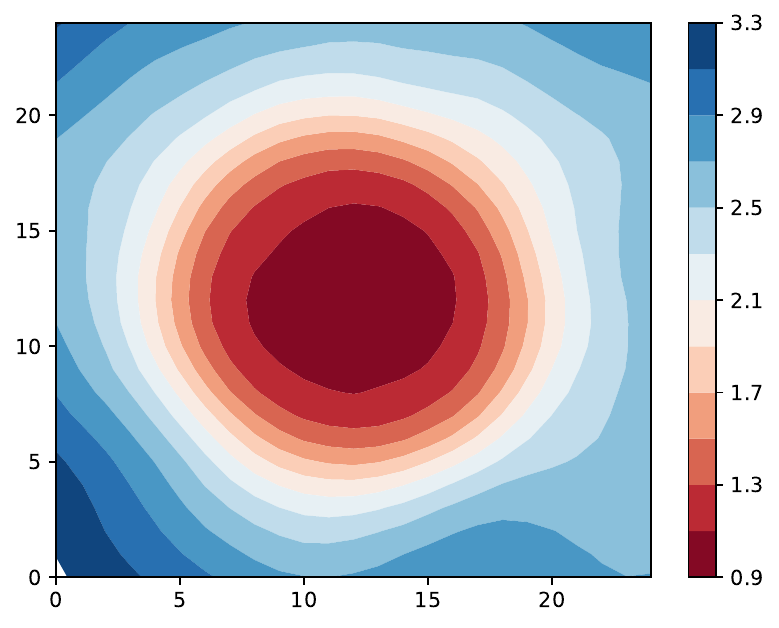}}
    \subfigure[Federated~($0.06$)]{
    \centering
    \label{fig:intro_3}
    \includegraphics[width=0.145\textwidth]{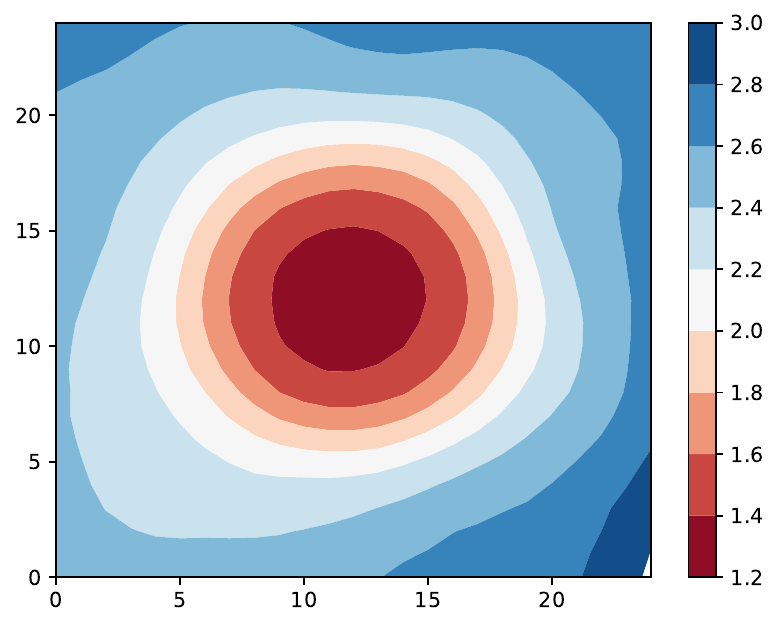}}
    \centering
    \subfigure[SAM in Local]{
    \centering
    \label{fig:intro_4}
    \includegraphics[width=0.17\textwidth]{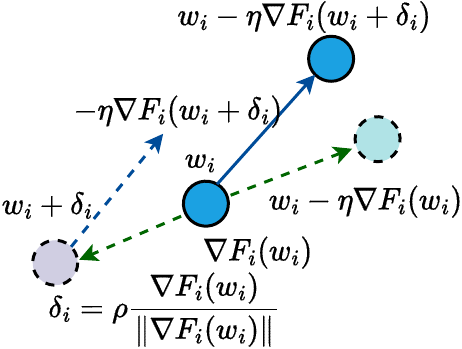}}
    \centering
    \subfigure[Conflicts]{
    \centering
    \label{fig:intro_5}
    \includegraphics[width=0.135\textwidth]{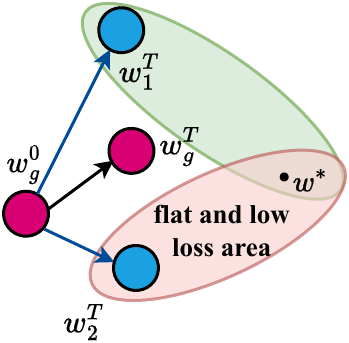}}
    \subfigure[FedLESAM]{
    \centering
    \label{fig:intro_6}
    \includegraphics[width=0.135\textwidth]{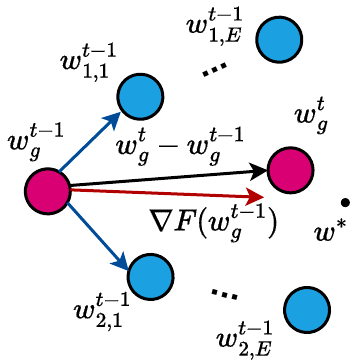}}
    \caption{Figures~\ref{fig:intro_1}-\ref{fig:intro_3} illustrate the loss surface for centralized training and federated training under Dirichlet distributions with coefficients of 0.6 and 0.06. Figure~\ref{fig:intro_4} depicts the local update process of FedSAM, including calculating perturbation based on client data and updating the local model using the gradient of the model after perturbation. Figure~\ref{fig:intro_5} highlights the sharpness minimizing conflicts due to discrepancies between local and global loss landscapes caused by data heterogeneity. Figure~\ref{fig:intro_6} demonstrates our locally estimating global perturbation~(opposite direction of red arrow) via global update~(opposite direction of black arrow).}\label{fig:intro}
    \vspace{-15pt}
\end{figure}

\begin{table*}[t!]
\caption{Summary of federated SAM-based algorithms for solving data heterogeneity, focusing on base algorithm, sharpness minimization target, perturbation calculation strategies, and extra computation introduced by SAM. In FedSMOO, $\mu_i$ and $s$ are dual variable and correction to perturbations. In FedLESAM, $w_i^\mathrm{old}$ is the global model received at previous active round. Refer Sec.~\ref{sec:notation} for other notations.}\label{tab:related}
\vspace{2pt}
\small
\centering
\renewcommand\arraystretch{0.01}
\setlength{\tabcolsep}{2pt}
\scalebox{0.95}{
\begin{tabular}{ c | c c cc}
\toprule[2pt]  \textbf {Research work} & \textbf {Base Algorithm} & \textbf {Minimizing Target} &\textbf {Local Perturbation} & \textbf {Extra Computation}\\ \midrule
\text { FedSAM~(ECCV22, ICML22) } & FedAvg&Local Sharpness & $\rho\frac{\nabla F_i(w_{i, k}^t)}{\|\nabla F_i(w_{i, k}^t\|}$ & $\checkmark$\\
\cmidrule{1-5}  \text { MofedSAM~(ICML22) } &FedAvg with Momentum &Local Sharpness &$\rho\frac{\nabla F_i(w_{i, k}^t)}{\|\nabla F_i(w_{i, k}^t\|} $ & $\checkmark$ \\
\cmidrule{1-5} \text { FedGAMMA~(TNNLS23) }& Scaffold  &  Local Sharpness& $\rho\frac{\nabla F_i(w_{i, k}^t)}{\|\nabla F_i(w_{i, k}^t\|}$ & $\checkmark$\\
\cmidrule{1-5} \text { FedSMOO~(ICML23) }&  FedDyn &  Local Sharpness with Correction&$\rho\frac{\nabla F_i(w_{i, k}^t)-\mu_i-s}{\|\nabla F_i(w_{i, k}^t)-\mu_i-s\|}$  & $\checkmark$\\
\cmidrule{1-5} \text { FedLESAM~(Ours)} & FedAvg, Scaffold, FedDyn &Global Sharpness&$\rho\frac{w_i^\mathrm{old}-w^t}{\|w_i^\mathrm{old}-w^t\|}$ & $\times$\\
\bottomrule[2pt]
\end{tabular}
}
\vspace{-9pt}
\end{table*}

\section{Introduction}
\label{sec:intro}
Federated Learning (FL) enables clients to collaboratively train a global model with a server without sharing their private data. 
As a representative paradigm in FL, FedAvg~\cite{fedavg} reduces the parameter transmission cost by increasing local training steps, which has drawn considerable attention in many fields such as medical diagnosis~\cite{medical1,medical2} and autonomous driving~\cite{driving1,AD}. 
However, challenges arise due to data heterogeneity and multi-step local updates~\cite{noniid1,noniid2,noniid3,fedskip,fedmr}, which often forms a sharper global loss landscape and leads the global model to converge to a sharp local minimum~\cite{fedsamicml,fedsameccv,fedsmoo,fedgamma}.
It is widely observed that such a sharp minimum tends to behave poor generalization ability~\cite{first_flatness_nips1994,sharp_minima_icml_2017,loss_landscape_nips_2018,domainsharp}.
As depicted in Figure~\ref{fig:intro_1}-\ref{fig:intro_3}, the loss surface in centralized training is substantially flatter compared to that in federated training and an increase in data heterogeneity sharpens the loss landscape, exacerbating performance degradation.

To address this challenge, recent innovations have leveraged sharpness-aware minimization (SAM)~\cite{sam} to find a flat minimum for better generalization by minimizing the loss of the model after perturbation. 
\citet{fedsameccv} and \citet{fedsamicml} pioneered SAM in FL and proposed FedSAM. 
\citet{fedsamicml} proposed a variant of FedSAM called MoFedSAM by adding local momentum. 
\citet{fedgamma} proposed FedGAMMA, which enhanced FedSAM by integrating variance reduction of Scaffold~\cite{scaffold}. 
Nevertheless, a common limitation persists: they all compute perturbations to minimize sharpness based on client data.
In heterogeneous scenarios, the local loss surfaces may not accurately reflect the flatness of the global loss surface. 
Therefore, minimizing local sharpness in these manners may not effectively guide the aggregated model to a global flat minimum.

In the process of minimizing local sharpness, as FedSAM illustrated in Figure~\ref{fig:intro_4}, clients follow a two-step procedure: 1) calculate local perturbations based on local gradients; 2) update their models using gradients computed on the model after perturbation.
However, the discrepancy between local and global loss surfaces becomes evident under heterogeneous data. 
As depicted in Figure~\ref{fig:intro_5}, the local perturbations, tailored to client data, guide client models toward their respective local flat minima ($w_1^*$ and $w_2^*$), which may significantly diverge from the global flat minimum ($w^*$). 
\citet{fedsmoo} noticed the difference and proposed FedSMOO to both correct local updates and the local perturbations. 
However, like other SAM-based methods, FedSMOO introduces many computational overheads, increasing the expenses of clients. 
We have summarized all SAM-based federated methods for solving data heterogeneity in Table~\ref{tab:related}.

In this study, we analyze that, to align the efficacy of SAM in FL with centralized training, it is essential to ensure the consistency between local and global updates and between local and global perturbations.
The former guarantees to minimize an upper bound of global sharpness and can be solved by incorporating previous research for eliminating client drifts~\cite{scaffold,fedyn}.
Therefore, the challenge remains in correctly estimating global perturbation, the direction of which is parallel with global gradient. 
As illustrated in Figure~\ref{fig:intro_6}, the global gradient~(red arrow) can be inferred from the global update~(black arrow), a strategy also employed in Scaffold to correct client updates. 
Inspired by this, we propose \textbf{FedLESAM}, a novel and efficient approach that \textbf{L}ocally \textbf{E}stimates global perturbation for \textbf{SAM} as the difference between global models received in the previous active and current rounds without extra computational overheads.
Empirically, we validate the local estimation of global perturbation and conduct comprehensive experiments to show the performance and efficiency. 
Theoretically, we provide the convergence guarantee of FedLESAM and prove a slightly tighter bound than FedSAM. 
Our contributions are threefold:
\begin{itemize}[leftmargin=15pt]
\vspace{-5pt}
	\item We rethink existing federated SAM-based algorithms for handling heterogeneous data, dissect the conflicts when minimizing local sharpness and analyze the conditions under which SAM is effective in FL~(Sec.~\ref{sec:rethink}).
 \vspace{-7pt}
	\item We present FedLESAM, a novel and efficient algorithm that minimizes global sharpness and reduces computational demand by locally estimating the global perturbation at the client level~(Sec.~\ref{sec:fedgesam}). Theoretically, we provide the convergence guarantee of FedLESAM and prove a slightly tighter bound than its original FedSAM~(Sec.~\ref{sec:analysis}).
     \vspace{-15pt}
	\item Empirically, we conducted comprehensive experiments on four benchmark datasets under three partition strategies to show the superior performance and the efficiency and ability to minimize global sharpness~(Sec.~\ref{sec:exp}).
 \vspace{-7pt}
\end{itemize}

\section{Preliminaries} \label{sec:method}
This section shows basic notations, definitions of SAM, and FedAvg. See Appendix~\ref{app:related} for the detailed related works.

\subsection{Basic Notations}\label{sec:notation}
The basic notations used in the paper are outlined as follows:
\begin{itemize}[leftmargin=10pt]
 \vspace{-9pt}
    \item $i, k, t$: Sequence number of client, local iteration within a round and the communication round, respectively.
     \vspace{-2pt}
    \item $\eta_\mathrm{l}, \eta_\mathrm{g}$: Local and global learning rate, respectively.
     \vspace{-2pt}
    \item $P(x,y)$, $P_i(x,y)$: Data distributions of the global and the $i$-th client, and satisfies $P(x,y)=\mathbb{E}_i P_i(x,y)$.
     \vspace{-3pt}
    \item $\xi$, $\xi_i$: One random variable $(x, y)$ sampled from $P(x,y)$ or $P_i(x,y)$, respectively.
     \vspace{-2pt}
    \item $w$, $w^t$, $w^t_{i,k}$: Model weights and weights of the global and local models of $i$-th client at $k$-th iteration in t-th round.
     \vspace{-2pt}
    \item $\mathcal{L}, \mathcal{L}(w,\xi)$: Loss function and specific loss of a sample.
     \vspace{-2pt}
    \item $F(w)$, $F_i(w)$: Expected loss under $w$ in the global distribution and in the client distribution, respectively.
     \vspace{-2pt}
    \item $\delta, \rho$: Perturbation towards to the sharpest point near the neighborhood of $w$, and the  pre-defined magnitude of $\delta$.
\end{itemize}
\subsection{Sharpness and SAM} \label{sec:intro_sam}
\textbf{Sharpness.}
Sharpness~\cite{keskar2017on} at $w$ with a loss function $\mathcal{L}$ and data distribution $P(x,y)$ can be defined as
\begin{equation}
    s(w, P) \triangleq \max_{\| \delta\|_2 \leq \rho} \mathbb{E}_{\xi \sim P(x,y)}[\mathcal{L}(w+\delta;\xi) - \mathcal{L}(w;\xi)]. \nonumber
\end{equation}
\textbf{SAM.} Many studies~\citep{first_flatness_nips1994,loss_landscape_nips_2018,sharp_minima_icml_2017} have demonstrated that a flat minimum tends to exhibit superior generalization ability in deep learning models and \citet{sam} proposed a sharpness-aware minimization~(SAM) as
\begin{equation}
   \min_w F^\mathrm{SAM}(w)=\min_w \max_{\| \delta\|_2 \leq \rho} \mathbb{E}_{\xi\sim P(x,y)} \mathcal{L}(w+\delta;\xi). \nonumber
\end{equation}
SAM minimizes both the sharpness and loss in two steps: 1) calculate perturbation as $\delta=\rho \frac{\nabla F(w)}{\|\nabla F(w)\|}$; 2) update the model with the gradient calculated after perturbation as $w=w-\eta \nabla F(w+\delta)$, where $\eta$ is the learning rate.

\subsection{Federated Learning via FedAvg} \label{sec:fedavg}
As shown in Algorithm~\ref{alg:fedgesam}, the vanilla FL via FedAvg~\cite{fedavg} consists of four steps: 1) In round $t$, the server distributes the global model $w^{t}$ to active $K$ clients; 2) Active clients receive and continue to train the model, \textit{e.g.,} the $i$-th client conducts the local training as $w_{i,k+1}^t\leftarrow w_{i,k}^t - \eta_\mathrm{l} \nabla \mathcal{L}(w_{i,k}^t,b_{i,k}^t)$, where $b_{i,k}^t$ is a batch of data and $k=0,...,E-1$; 3) After $E$ steps, the updated models are then communicated to the server; 4) The server performs the aggregation to acquire a new global model as $w^{t+1}\leftarrow  w^{t}-\eta_\mathrm{g}\frac{1}{K}\sum_{i=1}^K (w^{t}-w^{t}_{i,E}),$ where $K$ is the number of active clients in round $t$. When maximal round $T$ reaches, we will have the final optimized model $w^T$. 
\section{Rethink SAM in FL}
\label{sec:rethink}
This section delves into the analysis on when SAM works in FL, related works, a verification on the sharpness minimizing discrepancy, and our motivation. 
\subsection{When SAM Works in FL and Recent Works} \label{sec:how_work}
Given $i$-th client data distribution $P_i(x,y)$ and global distribution $P(x,y)$ with the relationship $P(x,y)=\mathbb{E}_i P_i(x,y)$, the SAM objective in centralized training is defined as follows~\cite{sam}:
\begin{equation}\label{eq:central}
\max_{\| \delta\| \leq \rho} \mathbb{E}_{\xi\sim P}\mathcal{L}(w+\delta;\xi)=\max_{\| \delta\| \leq \rho}\mathbb{E}_i \mathbb{E}_{\xi_i\sim P_i}\mathcal{L}(w+\delta;\xi_i).
\end{equation}
Constrained by the communication during the multi-step local updates, prevalent federated approaches integrate SAM into the local training~\cite{fedsameccv,fedsamicml,fedgamma,fedsmoo}. The SAM objective in FL is then formulated as
\begin{equation}
\mathbb{E}_i \max_{\| \delta_i\| \leq \rho} \mathbb{E}_{\xi_i\sim P_i}\mathcal{L}(w_i+\delta_i;\xi_i) \label{eq:conflict},
\end{equation}
where $\delta_i$ and $w_i$ are $i$-th client's perturbation and model weights. When client models are aligned in local updates, the objective of Equation~\ref{eq:conflict} is an upper bound of Equation~\ref{eq:central}:
\begin{equation}
\mathbb{E}_i \max_{\| \delta_i\| \leq \rho} \mathbb{E}_{\xi_i\sim P_i}\mathcal{L}(w+\delta_i;\xi_i)
\geq \max_{\| \delta\| \leq \rho}\mathbb{E}_i \mathbb{E}_{\xi_i\sim P_i}\mathcal{L}(w+\delta;\xi_i),
\nonumber
\end{equation}
where the inequality is from Jensen's inequality, specifically $\mathbb{E} [\max (x)]\geq \max (\mathbb{E} [x])$.
However, as the number of local updates and the degree of data heterogeneity increase, it becomes more difficult to maintain consistency of the global model with the client models.
In this case, minimizing local sharpness can not effectively achieve a global flat minimum.

Recent works, FedSAM~\cite{fedsameccv,fedsamicml}, MoFedSAM~\cite{fedsamicml}, and FedGAMMA~\cite{fedgamma}, all did not address this intrinsic discrepancy while MoFedSAM and FedGAMMA might mitigate this by introducing local momentum and variance reduction to prevent client drifts. 
FedSMOO~\cite{fedsmoo} noticed the difference and added a regularizer as FedDyn~\cite{fedyn} to correct both client updates and perturbations. 

\subsection{Verification and Motivation.} \label{sec:motivation}
To demonstrate the conflicts with heterogeneous data, we conducted experiments on CIFAR10 under the Dirichlet distribution with a coefficient of 0.1 and traced the global sharpness. 
As shown in the right panel of Figure~\ref{fig:motivation}, compared to FedAvg, FedSAM could not achieve satisfactory global flatness while MoFedSAM, FedGAMMA and FedSMOO obtained smaller sharpness but are still far away from our FedLESAM and the centralized training.
To further align the efficacy in FL~(Equation~\ref{eq:conflict}) with centralized training~(Equation~\ref{eq:central}), aside from increasing communication frequency (which raises communication costs), strategies for effectively minimizing global sharpness in FL involve reducing inconsistencies in client updates and estimating global perturbations in clients. 
The former that guarantees to minimize an upper bound of the global sharpness can be achieved by incorporating previous research such as Scaffold~\cite{scaffold} and FedDyn~\cite{fedyn}. 
Therefore, the challenge remains in correctly estimating the global perturbation in clients.
FedSMOO~\cite{fedsmoo} attempted to address this by correcting local perturbations, but it introduces many computational overheads as other SAM-based algorithms, which increases the expenses of clients in the federation.
To effectively optimize global sharpness and reduce the computational burden on clients, we propose a novel and efficient algorithm called FedLESAM and design two effective variants based on the frameworks Scaffold and FedDyn. 
FedLESAM locally estimates the direction of global perturbation on the client side as the difference between global models received in the previous active and the  current rounds without extra computation.

\begin{figure}[t!]
\centering  %图片全局居中
\subfigure{
% \label{fig:svhng}
% \hspace{-5.2pt}
\includegraphics[width=0.23\textwidth]{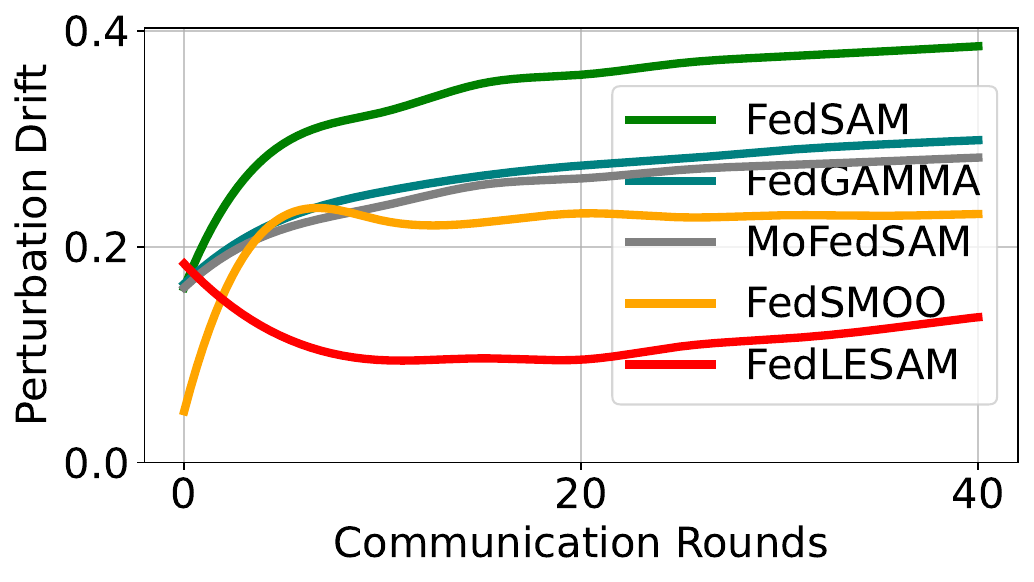}}\vspace{-8pt}
\subfigure{
% \label{fig:cifar10g}
\includegraphics[width=0.23\textwidth]{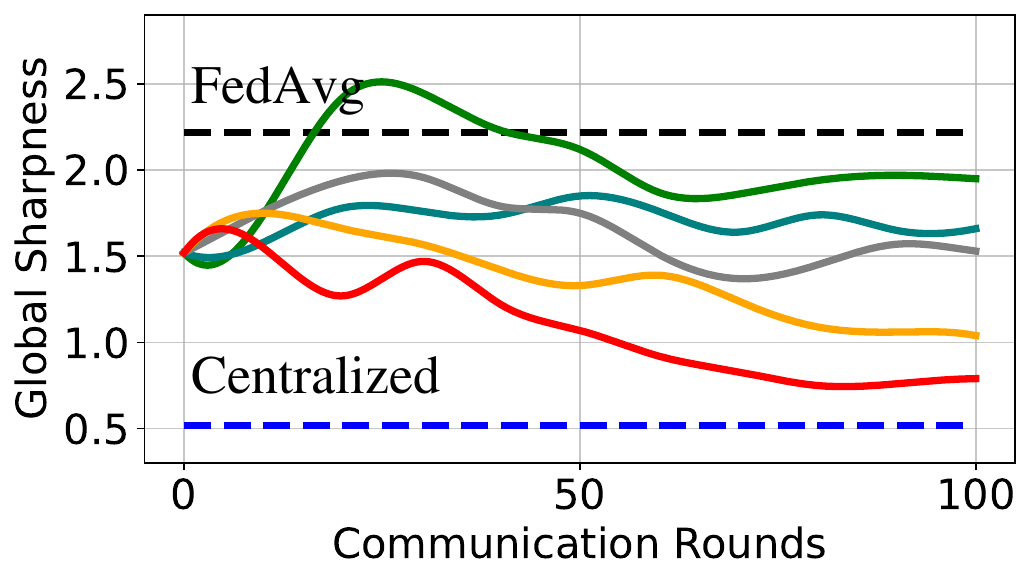}}
\caption{Illustration of perturbation drift~(left) ranged from 0 to 1 and global sharpness~(right) during federated training. The experiment was conducted on CIFAR10 under the Dirichlet distribution with coefficient of 0.1 with 100 clients and active ratio of 10\%. }
\vspace{-10pt}
\label{fig:motivation}
\end{figure}

\section{Method: FedLESAM}
\label{sec:fedgesam}
This section introduces our method FedLESAM with some primary analysis on the reasonableness, and demonstrates the total framework followed by two enhanced variants. 

\subsection{Efficiently Estimate Global Perturbation on Client}
As motivated above, our goal is to efficiently estimate global perturbation at each client without incurring additional computation overheads. 
To realize this, we first recall the definition of global sharpness-aware minimization in FL: 
\begin{equation}
\min_{w} \max_{\left\|\delta \right\|_2 \leq \rho}\left\{F(w+\delta)=\frac{1}{N} \sum_{i =1}^N F_i(w+\delta)\right\},
\nonumber
\end{equation}
where N is the number of clients. 
Without considering the communication frequency of local model weights between clients, we can obtain the virtual global perturbation $\delta_{\mathrm{g},k}^t$ at the $k$-th iteration in round $t$ as follows:
\begin{equation}
\delta_{\mathrm{g},k}^t=\rho\frac{\nabla F(w_{\mathrm{g},k}^t)}{\|\nabla F(w_{\mathrm{g},k}^t)\|}=\rho\frac{\sum_{i=1}^N \nabla F_i(w_{\mathrm{g},k}^t)}{\|\sum_{i=1}^N \nabla F_i(w_{\mathrm{g},k}^t)\|},
\nonumber
\end{equation}
where $w^t_{\mathrm{g},k}=w^t-\eta_\mathrm{g} \frac{1}{N}\sum_{i=0}^N (w^t-w_{i,k}^t) $ is the virtual global model. 
However, we can neither share weights nor gradients of clients during the local training. 
An alternative way is to estimate the global gradient at clients.
As illustrated in Figure~\ref{fig:intro_6}, we can estimate the global gradient in red color as the global updates in black color between two communication rounds $w^{t}-w^{t-1}$. 
Such estimation strategy is also applied in Scaffold~\cite{scaffold} as a global update to correct client updates, which is introduced in Appendix~\ref{app:variants}.
Under straggler situations, clients might not be active at all rounds and obtain $w^{t-1}$. 
Since the communication with server for $w^{t-1}$ increases communication cost, clients can utilize the global model received in the previous active round, denoted as $w^\mathrm{old}_i$. 
Therefore, the global perturbation can be approximately calculated as follows:
\begin{equation}\label{eq:estimate}
\delta_{\mathrm{g},k}^t=\rho\frac{\nabla F(w_{\mathrm{g},k}^t)}{\| \nabla F(w_{\mathrm{g},k}^t)\|}\approx \rho\frac{w^\mathrm{old}_i-w^t}{\|w^\mathrm{old}_i-w^t\|}.
\end{equation}
Notably, here we utilize $w^\mathrm{old}_i-w^t$ to estimate the direction of global gradient and the scaling issue from previous iteration to current iteration is addressed in the calculation of perturbation $\frac{w^\mathrm{old}_i-w^t}{\|w^\mathrm{old}_i-w^t\|}$. Under full participation or permitted to communicate last-round global model with server, $w^{old}$ will be equal to $w^{t-1}$.
\emph{For practical, we forbid such communication in the experiments.} Finally, we define the update of our FedLESAM that locally estimates global perturbation for SAM as
\begin{equation}
w_{i,k+1}^t=w_{i,k}^t-\eta_\mathrm{l}\nabla F_i(w_{i,k}^t+\rho\delta_{\mathrm{g},k}^t).\nonumber
\end{equation}

\textbf{Reasonableness.} Our estimation is possible, if the direction of ascent step on data sampled from general distribution $P$ can be inferred by global updates: $\nabla F(w_g^t,\xi\in P)\approx C\Delta w_g^t=C'(w^{t-1}-w^{t})\approx C''(w^{old}-w^{t}),$
where C, C' and C'' are constant values. This strategy is also applied in Scaffold to estimate global desent step. To show the reasonableness of the estimation on global perturbation, here we provide some primary analysis. 
In Section~\ref{sec:error}, we provide the estimation bias under one local update and full participation. 
The error can be bounded and influenced by the smoothness of the global loss function, learning rates, data heterogeneity, and sampling in stochastic gradient. 
To reduce the bias, we could set proper global and local learning rates. 
Empirically, we conducted an experiment on CIFAR10 and traced the global sharpness and perturbation drifts~(PD). 
A PD is value to estimate bias between local and global perturbations defined as $\text{PD}^t=\frac{1}{2KE}\sum_{k=0}^{E-1}\sum_{i=1}^K\|\delta^t_{\mathrm{g},k}-\delta_{i,k}^t\|$,
where $K$ is the number of active clients, $\delta^t_{\mathrm{g},k}$ is the global perturbation, and $\delta_{i,k}^t$ is the local perturbation. 
As shown in Figure~\ref{fig:motivation}, the PD value and global sharpness of our FedLESAM are much smaller than others, which verifies the effectiveness of our method in estimating the global perturbation and the superior ability to minimize global sharpness.

\subsection{Total Framework} 
The overall framework is summarized in Algorithm~\ref{alg:fedgesam}.
At the perturbation stage, clients use the difference between global model received in the last active round $w_i^\mathrm{old}$ and newly received global model $w^t$ as the direction of the global perturbation throughout the local training.
Then, all selected clients calculate the gradient after perturbation and perform local updates. 
At the end of local training, all local clients update the $w_i^\mathrm{old}$ as the originally received global model. 
The key distinction between FedAvg, FedSAM, and FedLESAM lies in the perturbation stage, highlighted in Algorithm~\ref{alg:fedgesam}. 
Unlike FedAvg, FedSAM calculates the perturbations as local gradients, while our FedLESAM leverages $w_i^\mathrm{old}$ to estimate global perturbation, reducing computational demands. 
Other parts in FedSAM and FedLESAM such as aggregation and communication are the same as FedAvg.

\subsection{Enhanced Variants} 
The global perturbation in our method is estimated as the difference between global models received in the previous active and the current rounds throughout the local training. 
When the global update is changing fast or the local models are far away from each other, the estimated perturbation might not be accurate. 
Therefore, to eliminate the inconsistencies for better estimation and a fair comparison with FedGAMMA and FedSMOO, we incorporate the variance reduction of Scaffold and dynamic regularizer of FedDyn into FedLESAM and propose two variants named FedLESAM-S and FedLESAM-D. In Appendix~\ref{app:variants}, we introduce them in detail and provide concrete algorithms.

\section{Theoretical Analysis}
\label{sec:analysis}
Generalization results proposed by~\citet{fedsamicml} and~\citet{fedsmoo} are both suitable for our FedLESAM.
Here we mainly focus on the convergence results of FedLESAM compared to its original FedSAM with an independent perturbation magnitude. 
The convergence results of our variants FedLESAM-S and FedLESAM-D can be easily extended.

\begin{algorithm}[t!]
    {\footnotesize
    \caption{FedAvg, FedSAM and FedLESAM}
    \label{alg:fedgesam}
    \textbf{Input}:$(K, \rho, w^0, E, T, \eta_\mathrm{l}, \eta_\mathrm{g}, \forall i~ w_i^\mathrm{old}=0)$
    \begin{algorithmic}
    \FOR{$t = 0,1,\dots,T-1$}
        \FOR{sampled $n$ active client $i = 1,2,\dots,n$}
            \STATE receive $w^t$, $w^t_{i,0}\leftarrow w^t$
            \FOR{$k=0,1,...,E-1$} 
                \STATE sample a batch of data $b_{i,k}^t$
                \STATE \colorbox{gray!20}{$\rhd$ perturbation stage}
                \STATE  \colorbox{green!20}{{\textcolor{blue}{FedAvg:~~~~~~~~}}$\delta^t_{i,k}=0$}
                \STATE  \colorbox{blue!20}{{\textcolor{blue}{FedSAM:~~~~~~}}$\delta^t_{i,k}=\rho \frac{\nabla\mathcal{L}\left(w^t_{i,k} ; b_{i,k}^t\right)}{\|\nabla\mathcal{L}\left(w^t_{i,k} ; b_{i,k}^t\right)\|}$ }
                \STATE \colorbox{red!20}{{\textcolor{blue}{FedLESAM:~}}$\delta^t_{i,k}=\rho \frac{w_i^\mathrm{old}-w^t}{\|w_i^\mathrm{old}-w^t\|}$ }
                % \STATE $\hat{w}_{i,k}^t=w_{i,k}^t+\delta^t_{i,k}$
                \STATE $w_{i,k+1}^{t} \leftarrow w_{i,k}^t - \eta_\mathrm{l} \nabla\mathcal{L}(w_{i,k}^t+\delta^t_{i,k} ; b_{i,k}^t)$
            \ENDFOR
                \STATE \colorbox{red!20}{{\textcolor{blue}{FedLESAM:~}}store $w_i^\mathrm{old}=w^t$}
                \STATE submit $w^{t}_{i,E}.$
        \ENDFOR
        \STATE $w^{t+1} \leftarrow w^t-\eta_\mathrm{g}\sum_{i=1}^{K}{w^t-w^t_{i,E}}$.
    \ENDFOR
    \end{algorithmic}
{\textbf{Output}:$w^T.$}
}
\end{algorithm} 

\subsection{Basic Assumptions}
We first introduce some basic assumptions on clients' loss functions~$F_1,\cdots, F_N$ and their gradients, which are the same as FedSAM~\cite{fedsamicml}. Assumptions~\ref{asm:smooth_var}-\ref{asm:grad_differ} characterize the smoothness, bound on the variance of unit stochastic gradients, and the bound on the gradient difference between local and global objectives, while Assumption~\ref{asm:unit_variance}-\ref{asm:unit_differ} cares more about the bounds under averaged situations.

\begin{assumption}[$L$-smooth and bounded variance of unit stochastic gradients] 
$F_1, \cdots, F_N$ are all $L$-smooth:
\begin{equation}
    \left\|\nabla F_i(u)-\nabla F_i(v)\right\| \leq L\|u-v\| \nonumber,
\end{equation}
and the variance of unit stochastic gradients is bounded:
\begin{equation}
\mathbb{E}\left\|\frac{\nabla F_i\left(u, \xi_i\right)}{\left\|\nabla F_i\left(u, \xi_i\right)\right\|}-\frac{\nabla F_i(u)}{\left\|\nabla F_i(u)\right\|}\right\|^2 \leq \sigma_\mathrm{l}^2. \nonumber
\end{equation}
\label{asm:smooth_var}
\end{assumption}
\begin{assumption}[Bounded heterogeneity] 
The gradient difference between $F(u)$ and $F_i(u)$ is bounded:
\begin{equation}
   \left\|\nabla F_i\left(u\right)-\nabla F\left(u\right)\right\|^2 \leq \sigma_\mathrm{g}^2 \nonumber
\end{equation}
\label{asm:grad_differ}
\vspace{-14pt}
\end{assumption}

\begin{assumption}[Bounded unit variance] 
Variance of unit averaged stochastic gradients is bounded:
\begin{equation}
\mathbb{E}\left\|\frac{\sum_{i=1}^N \nabla F_i\left(u, \xi_i\right)}{\left\|\sum_{i=1}^N \nabla F_i\left(u, \xi_i\right)\right\|}-\frac{\sum_{i=1}^N \nabla F_i\left(u\right)}{\left\|\sum_{i=1}^N \nabla F_i\left(u\right)\right\|}\right\|^2\leq \sigma_\mathrm{l}'^2. \nonumber
\end{equation}
\label{asm:unit_variance}
\vspace{-14pt}
\end{assumption}

\begin{assumption}[Bounded unit difference] 
The variance of unit averaged gradient difference between $F(u)$ and $\sum_{i=1}^NF_i(u)$ is bounded:
\begin{equation}
\frac{\sum_{i=1}^N \nabla F_i\left(u\right)}{\left\|\sum_{i=1}^N \nabla F_i\left(u\right)\right\|}-\frac{\nabla F\left(u\right)}{\| \nabla F\left(u\right)\|}\leq \sigma_\mathrm{g}'^2. \nonumber
\end{equation}
\label{asm:unit_differ}
\vspace{-14pt}
\end{assumption}

\subsection{Convergence Results and Trade-off}
% except for the perturbation drifts, 
\begin{theorem}
    \label{thm:gefedsam}
    Let Assumption~\ref{asm:smooth_var}-\ref{asm:grad_differ} hold, with an independent $\rho$ under full participation, if choosing $\eta_\mathrm{l}=\frac{1}{\sqrt{T} E L}$ and $\eta_\mathrm{g}=\sqrt{E N}$, the sequence of $\{w^t\}$ generated by FedSAM and FedLESAM in Algorithm~\ref{alg:fedgesam} satisfies:
    \begin{equation}
    \begin{aligned} 
    &\frac{1}{T} \sum_{t=1}^T \mathbb{E}\left[\left\|\nabla F\left(w^{t+1}\right)\right\|\right]\leq \frac{10L(F\left(w^0\right)-F^*)}{C\sqrt{T E N}}\\  
    & +\frac{90L^2\rho^2\sigma_\mathrm{g}^2}{CTE}+\frac{180 L^2\rho^2}{C T}+\Delta+\frac{L^2\sigma_\mathrm{l}^2 \rho^2}{C\sqrt{T E N}}, \nonumber
    \end{aligned}
    \end{equation}
    where $C \geq (\frac{1}{2}-30 E^2 L^2 \eta_\mathrm{l}^2) \geq 0$. For FedSAM, $\Delta=\frac{120L^2\rho^2}{CET^2}+\frac{2L^2\rho^2}{CT}$, while for our FedLESAM, $\Delta=0$.
\end{theorem}
As shown in Figure~\ref{fig:intro_5}, local perturbations might guide the aggregated global model far away from the global flat minimum. 
Therefore, we keep $\rho$ as an independent constant and provide the updated convergence results of FedSAM and our FedLESAM under full client participation as shown in Theorem~\ref{thm:gefedsam}. 
It can be seen that, by replacing the local perturbation of FedSAM with our locally estimated global perturbation, the convergence bound can be reduced by a rate of $\Delta=\frac{120L^2\rho^2}{CET^2}+\frac{2L^2\rho^2}{CT}$. 
The complete proof is provided in Appendix~\ref{app:proof}. 
Notably, the independent perturbation magnitude $\rho$ will influence the largest term $\mathcal{O}\left(\frac{1}{\sqrt{T}}\right)$ in the convergence bound as $\frac{L^2\sigma_\mathrm{l}^2 \rho^2}{C\sqrt{T E N}}$.
To mitigate the influence, all existing convergence theorems ~\cite{fedsamicml,fedgamma,fedsmoo} require the perturbation magnitude $\rho$ be a scale of total rounds like $\mathcal{O}\left(\frac{1}{\sqrt{T}}\right)$. 
However, as generalization results analyzed by~\citet{sam}, \citet{fedsamicml}, and ~\citet{fedsmoo}, $\rho$ is highly related to the generalization error bound. 
Note that, those generalization results are commonly suitable for federated SAM algorithms, including our FedLESAM. 
Therefore, the chosen of $\rho$ will be a significant trade-off between the generalization and convergence. 
In the ablation study of Sec.~\ref{sec:exp} and as shown in Figure~\ref{fig:ablation}, we empirically verify the relationships. 

\subsection{Estimation Error}\label{sec:error}
\begin{theorem}
    \label{thm:naive}
    Assume local update is one step and follows Assumptions~\ref{asm:unit_variance}-~\ref{asm:unit_differ}. Under full participation and $L_\mathrm{g}$-smoothness of $F$ with global and local learning rates $\eta_\mathrm{g}$ and $\eta_\mathrm{l}$, the estimation bias is bounded as
    \begin{equation}
    \|\frac{w^{t-1}-w^t}{\|w^{t-1}-w^t\|}-\frac{\nabla F(w^t)}{\|\nabla F(w^t)\|}\| \leq 3\sigma_\mathrm{l}'^2+3\sigma_\mathrm{g}'^2+3L_\mathrm{g}^2 \eta_\mathrm{g}^2\eta_\mathrm{l}^2.
    \nonumber
    \end{equation}
\end{theorem}
As shown in Theorem~\ref{thm:naive}, we provide the estimation error bound under one step local update and full participation. 
With Assumption~\ref{asm:unit_variance}-\ref{asm:unit_differ}, the estimation error can be bounded and is influenced by learning rates $\eta_\mathrm{g}$ and $\eta_\mathrm{l}$, smoothness of global function $L_\mathrm{g}$, sampling in stochastic gradient ($\sigma_\mathrm{l}'^2$) and data heterogeneity ($\sigma_\mathrm{g}'^2$).
With this insights, we can reduce the error by providing proper learning rates.
The detailed proof is provided in Appendix~\ref{app:proof}. 

\begin{table*}[th!]
\setlength{\tabcolsep}{5pt}
\centering
\caption{Test accuracy on CIFAR10/100 after 800 rounds under Dirichlet distribution and Pathological splits. $\beta$ is the Dirichlet coefficient selected from $\{0.1, 0.6\}$ and $\alpha$ is the Pathological coefficient, which is the number of active categories in each client. The two datasets are divided into 100 clients and 10\% of them are active at each round in the upper part, while 200 and 5\% in the lower part.}
\vspace{2pt}
\renewcommand\arraystretch{0.8}
\small
\centering
\scalebox{0.96}{
\begin{tabular}{l|cccc|cccc}
\toprule[2pt]
Method & \multicolumn{4}{c|}{ CIFAR10} & \multicolumn{4}{c}{ CIFAR100} \\
\cmidrule{1-9} \#Partition & \multicolumn{2}{c}{ Dirichlet } & \multicolumn{2}{c|}{Pathological} & \multicolumn{2}{c}{ Dirichlet }  & \multicolumn{2}{c}{Pathological} \\
\cmidrule{1-9} \#Coefficient& $\beta=0.6$ & $\beta=0.1$ & $\alpha=6$ & $\alpha=3$ & $\beta=0.6$ & $\beta=0.1$ & $\alpha=20$ & $\alpha=10$  \\
\midrule 
FedAvg & $79.52_{\pm 0.13}$ & $76.00_{\pm 0.18}$& $79.91_{\pm 0.17}$ & $74.08_{\pm 0.22}$ & $46.35_{\pm 0.15}$ & $42.64_{\pm 0.22}$ & $44.15_{\pm 0.17}$ & $40.23_{\pm 0.31}$  \\
FedAdam & $77.08_{\pm 0.31}$ & $73.41_{\pm 0.33}$& $77.05_{\pm 0.26}$ & $72.44_{\pm 0.29}$ & $48.35_{\pm 0.17}$ & $40.77_{\pm 0.31}$ & $41.26_{\pm 0.30}$ & $32.58_{\pm 0.22}$  \\
SCAFFOLD & $81.81_{\pm 0.17}$ & $78.57_{\pm 0.14}$ & $83.07_{\pm 0.10}$ & $77.02_{\pm 0.18}$ & $51.98_{\pm 0.23}$ & $44.41_{\pm 0.15}$& $46.06_{\pm 0.22}$ & $41.08_{\pm 0.24}$   \\
FedCM & $82.97_{\pm 0.21}$ & $77.82_{\pm 0.16}$ & $83.44_{\pm 0.17}$ & $77.82_{\pm 0.19}$ & $51.56_{\pm 0.20}$ & $43.03_{\pm 0.26}$& $44.94_{\pm 0.14}$ & $38.35_{\pm 0.27}$   \\
FedDyn  & $83.22_{\pm 0.18}$ & $78.08_{\pm 0.19}$ & $83.18_{\pm 0.17}$ & $77.63_{\pm 0.14}$ & $50.82_{\pm 0.19}$ & $42.50_{\pm 0.28}$& $44.19_{\pm 0.19}$ & $38.68_{\pm 0.14}$  \\
FedSAM& $80.10_{\pm 0.12}$ & $76.86_{\pm 0.16}$ & $80.80_{\pm 0.23}$ & $75.51_{\pm 0.24}$& $47.51_{\pm 0.26}$ & $43.43_{\pm 0.12}$ & $45.46_{\pm 0.29}$ & $40.44_{\pm 0.23}$   \\
MoFedSAM& $84.13_{\pm 0.13}$ & $78.71_{\pm 0.15}$ & $84.92_{\pm 0.14}$ & $79.57_{\pm 0.18}$& $\underline{54.38}_{\pm 0.22}$ & $44.85_{\pm 0.25}$ & $47.42_{\pm 0.26}$ & $41.17_{\pm 0.22}$   \\
FedGAMMA  & $82.64_{\pm 0.14}$ & $78.95_{\pm 0.15}$ & $83.24_{\pm 0.19}$ & $78.81_{\pm 0.14}$ & $53.41_{\pm 0.20}$ & $46.39_{\pm 0.19}$& $48.41_{\pm 0.14}$ & $43.24_{\pm 0.22}$  \\
FedSMOO  & $\underline{84.55}_{\pm 0.14}$ & $\textbf{80.82}_{\pm 0.17}$ & $\underline{85.39}_{\pm 0.21}$ & $\underline{81.58}_{\pm 0.16}$ & $53.92_{\pm 0.18}$ & $\underline{46.48}_{\pm 0.13}$& $\underline{48.87}_{\pm 0.17}$ & $\underline{44.10}_{\pm 0.19}$  \\
FedLESAM & ${81.04}_{\pm 0.19}$ & ${76.93}_{\pm 0.16}$ & ${81.37}_{\pm 0.17}$ & ${77.30}_{\pm 0.22}$ & ${47.92}_{\pm 0.19}$ & ${44.48}_{\pm 0.20}$& ${46.19}_{\pm 0.21}$ & ${41.20}_{\pm 0.18}$ \\
FedLESAM-D  & ${84.27}_{\pm 0.14}$ & ${80.08}_{\pm 0.19}$ & ${85.62}_{\pm 0.18}$ & ${83.00}_{\pm 0.22}$ & ${53.27}_{\pm 0.17}$ & ${46.42}_{\pm 0.23}$& ${48.26}_{\pm 0.18}$ & ${43.26}_{\pm 0.18}$  \\
FedLESAM-S  & $\color{red}\textbf{84.94}_{\pm 0.12}$ & $\underline{79.52}_{\pm 0.17}$ & $\color{red}\textbf{85.88}_{\pm 0.19}$ & $\color{red}\textbf{83.18}_{\pm 0.15}$ & $\color{red}\textbf{54.61}_{\pm 0.20}$ & $\color{red}\textbf{48.07}_{\pm 0.19}$& $\color{red}\textbf{50.26}_{\pm 0.18}$ & $\color{red}\textbf{44.42}_{\pm 0.17}$  \\
 \midrule 
FedAvg & $75.90_{\pm 0.21}$ & $72.93_{\pm 0.19}$& $77.47_{\pm 0.34}$ & $71.86_{\pm 0.34}$ & $44.70_{\pm 0.22}$ & $40.41_{\pm 0.33}$ & $38.22_{\pm 0.25}$ & $36.79_{\pm 0.32}$  \\
FedAdam & $75.55_{\pm 0.38}$ & $69.70_{\pm 0.32}$& $75.24_{\pm 0.22}$ & $70.49_{\pm 0.26}$ & $44.33_{\pm 0.26}$ & $38.04_{\pm 0.25}$ & $35.14_{\pm 0.16}$ & $30.28_{\pm 0.28}$  \\
 SCAFFOLD& $79.00_{\pm 0.26}$ & $76.15_{\pm 0.15}$ & $80.69_{\pm 0.21}$ & $74.05_{\pm 0.31}$ & $50.70_{\pm 0.18}$ & $41.83_{\pm 0.29}$& $39.63_{\pm 0.31}$ & $37.98_{\pm 0.36}$   \\
 FedCM & $80.52_{\pm 0.29}$ & $77.28_{\pm 0.22}$ & $81.76_{\pm 0.24}$ & $76.72_{\pm 0.25}$ & $50.93_{\pm 0.31}$ & $42.33_{\pm 0.19}$& $42.01_{\pm 0.17}$ & $38.35_{\pm 0.24}$   \\
 FedDyn  & $80.69_{\pm 0.23}$ & $76.82_{\pm 0.17}$ & $82.21_{\pm 0.18}$ & $74.93_{\pm 0.22}$ & $47.32_{\pm 0.18}$ & $41.74_{\pm 0.21}$& $41.55_{\pm 0.18}$ & $38.09_{\pm 0.27}$  \\
 FedSAM& $76.32_{\pm 0.16}$ & $73.44_{\pm 0.14}$ & $78.16_{\pm 0.27}$ & $72.41_{\pm 0.29}$& $45.98_{\pm 0.27}$ & $40.22_{\pm 0.27}$ & $38.71_{\pm 0.23}$ & $36.90_{\pm 0.29}$   \\
 MoFedSAM& $82.58_{\pm 0.21}$ & $78.43_{\pm 0.24}$ & $84.46_{\pm 0.20}$ & $79.93_{\pm 0.19}$& $\textbf{53.51}_{\pm 0.25}$ & $42.22_{\pm 0.23}$ & $42.77_{\pm 0.27}$ & $39.81_{\pm 0.21}$   \\
FedGAMMA  & $80.72_{\pm 0.19}$ & $76.41_{\pm 0.17}$ & $81.81_{\pm 0.17}$ & $76.58_{\pm 0.21}$ & $50.61_{\pm 0.19}$ & $43.77_{\pm 0.19}$& $43.35_{\pm 0.24}$ & $38.46_{\pm 0.22}$  \\
FedSMOO  & $\underline{82.94}_{\pm 0.19}$ & $\textbf{79.76}_{\pm 0.19}$ & $84.82_{\pm 0.18}$ & $\underline{81.01}_{\pm 0.19}$ & $\underline{53.45}_{\pm 0.19}$ & $\textbf{45.83}_{\pm 0.18}$& $\underline{44.70}_{\pm 0.21}$ & $\underline{43.41}_{\pm 0.22}$  \\
FedLESAM  & ${77.74}_{\pm 0.18}$ & ${73.73}_{\pm 0.22}$ & ${78.44}_{\pm 0.20}$ & ${74.53}_{\pm 0.19}$ & ${45.00}_{\pm 0.16}$ & ${41.87}_{\pm 0.23}$& ${42.14}_{\pm 0.18}$ & ${39.32}_{\pm 0.24}$  \\
FedLESAM-D  & ${82.53}_{\pm 0.19}$ & $\underline{79.56}_{\pm 0.27}$ & $\color{red}\textbf{85.04}_{\pm 0.21}$ & $\color{red}\textbf{81.10}_{\pm 0.19}$ & ${51.14}_{\pm 0.20}$ & $\underline{45.09}_{\pm 0.24}$& ${43.97}_{\pm 0.26}$ & ${42.63}_{\pm 0.29}$  \\
FedLESAM-S  & $\color{red}\textbf{83.22}_{\pm 0.22}$ & $78.69_{\pm 0.17}$ & $\underline{85.02}_{\pm 0.24}$ & ${80.57}_{\pm 0.17}$ & $52.26_{\pm 0.18}$ & $44.82_{\pm 0.20}$& $\color{red}\textbf{45.68}_{\pm 0.19}$ & $\color{red}\textbf{43.89}_{\pm 0.23}$  \\
\bottomrule[2pt]
\end{tabular}
}
\label{tb:cifar}
\vspace{-12pt}
\end{table*}

\begin{table*}[th!]
\centering
\caption{Accuracy of the target domain on OfficeHome and DomainNet after 400 rounds under leave-one-domain-out strategy. Each training domain is divided into 1 client and 100\% of them are active at each round in the upper part while 10 and 20\% in the lower part.}
\vspace{2pt}
\renewcommand\arraystretch{1.0}
\label{tb:domain}
\small
\centering
\setlength{\tabcolsep}{2pt}
\scalebox{0.92}{
\begin{tabular}{l|cccc|cccccc}
\toprule[2pt]
Method & \multicolumn{4}{c|}{ Officehome} & \multicolumn{6}{c}{ DomainNet} \\
\cmidrule{1-11} \#Target domain& Art & Clipart & Product & Real World& Clipart & Infograph & Painting & Quickdraw  & Real World& Sketch\\
\midrule 
FedAvg & $79.21_{\pm 0.17}$ & $60.60_{\pm 0.11}$& $86.22_{\pm 0.14}$ & $87.65_{\pm 0.14}$& $54.70_{\pm 0.11}$ & $81.59_{\pm 0.14}$ & $36.27_{\pm 0.27}$ & $76.49_{\pm 0.11}$ & $87.52_{\pm0.10}$ & $87.31_{\pm 0.13}$\\
FedAdam & $79.23_{\pm 0.23}$ & $61.21_{\pm 0.19}$& $86.00_{\pm 0.14}$ & $87.69_{\pm 0.12}$ & $\underline{56.77}_{\pm 0.25}$ & $81.33_{\pm 0.12}$ & $40.14_{\pm 0.24}$ & $78.43_{\pm 0.11}$& $87.46_{\pm 0.10}$& $88.22_{\pm 0.17}$  \\
SCAFFOLD& $80.35_{\pm 0.14}$ & $62.41_{\pm 0.13}$& $86.42_{\pm 0.14}$ & $\underline{88.39}_{\pm 0.11}$& $55.38_{\pm 0.17}$ & $82.28_{\pm 0.09}$ & $41.01_{\pm 0.24}$ & $77.26_{\pm 0.13}$ & $89.09_{\pm0.10}$ & $87.11_{\pm 0.14}$\\
FedCM& $80.10_{\pm 0.17}$& $61.10_{\pm 0.19}$ & $86.55_{\pm 0.17}$& $87.40_{\pm 0.17}$ & $55.30_{\pm 0.21}$ & $81.75_{\pm 0.17}$ & $38.98_{\pm 0.29}$ & $\underline{78.78}_{\pm 0.11}$ & $88.09_{\pm0.14}$ & $88.15_{\pm 0.14}$\\
FedDyn& $79.89_{\pm 0.17}$ & $56.27_{\pm 0.19}$& $84.97_{\pm 0.17}$ & $86.78_{\pm 0.16}$& $54.92_{\pm 0.21}$ & $80.72_{\pm 0.14}$ & $34.71_{\pm 0.27}$ & $77.69_{\pm 0.11}$ & $85.22_{\pm0.14}$ & $87.66_{\pm 0.16}$\\
FedSAM & $79.85_{\pm 0.14}$ & $62.25_{\pm 0.17}$& $86.71_{\pm 0.13}$ & $88.18_{\pm 0.16}$& $55.36_{\pm 0.14}$ & $82.20_{\pm 0.11}$ & $39.19_{\pm 0.20}$ & $77.53_{\pm 0.11}$ & $88.41_{\pm0.14}$ & $88.38_{\pm 0.09}$\\
MoFedSAM & $80.51_{\pm 0.14}$ & $62.47_{\pm 0.19}$& $86.80_{\pm 0.14}$ & $88.24_{\pm 0.11}$& $55.47_{\pm 0.17}$ & $82.33_{\pm 0.13}$ & $40.18_{\pm 0.26}$ & $78.43_{\pm 0.17}$ & $88.96_{\pm0.10}$ & $\underline{89.16}_{\pm 0.16}$\\
FedGAMMA& $\underline{80.63}_{\pm 0.17}$ & $\underline{62.68}_{\pm 0.19}$& $\underline{86.82}_{\pm 0.14}$ & $88.32_{\pm 0.17}$& $55.45_{\pm 0.20}$ & $\underline{82.55}_{\pm 0.14}$ & $\underline{41.10}_{\pm 0.23}$ & $77.30_{\pm 0.11}$ & $\underline{89.17}_{\pm0.09}$ & $87.54_{\pm 0.14}$\\
FedSMOO& $80.42_{\pm 0.17}$ & $57.77_{\pm 0.21}$& $85.43_{\pm 0.16}$ & $87.84_{\pm 0.19}$& $53.61_{\pm 0.24}$ & $81.99_{\pm 0.17}$ & $37.29_{\pm 0.34}$ & $77.92_{\pm 0.19}$ & $86.73_{\pm0.14}$ & $87.82_{\pm 0.19}$\\
FedLESAM& ${79.55}_{\pm 0.19}$ & ${60.57}_{\pm 0.16}$& ${86.49}_{\pm 0.21}$ & ${87.30}_{\pm 0.14}$& ${55.47}_{\pm 0.17}$& ${82.04}_{\pm 0.16}$ & ${39.86}_{\pm 0.12}$ & ${77.42}_{\pm 0.20}$ & ${87.63}_{\pm 0.19}$ & ${86.94}_{\pm0.11}$\\
FedLESAM-D& ${78.85}_{\pm 0.15}$ & ${57.34}_{\pm 0.12}$& ${85.62}_{\pm 0.11}$ & ${86.99}_{\pm 0.10}$& ${54.75}_{\pm 0.18}$ & ${82.24}_{\pm 0.16}$ & ${37.98}_{\pm 0.27}$ & ${77.54}_{\pm 0.22}$ & ${87.12}_{\pm0.19}$ & ${87.54}_{\pm 0.17}$\\

FedLESAM-S& ${\color{red}\textbf{81.10}_{\pm 0.14}}$  & ${\color{red}\textbf{62.86}_{\pm 0.16}}$ & ${\color{red}\textbf{87.34}_{\pm0.14}}$ & ${\color{red}\textbf{89.04}_{\pm 0.10}}$& ${\color{red}\textbf{57.24}_{\pm 0.19}}$& ${\color{red}\textbf{83.15}_{\pm 0.14}}$ & ${\color{red}\textbf{43.49}_{\pm 0.17}}$& ${\color{red}\textbf{79.31}_{\pm 0.10}}$ & ${\color{red}\textbf{89.26}_{\pm 0.09}}$ & ${\color{red}\textbf{89.61}_{\pm 0.14}}$\\
 \midrule 
FedAvg & $78.41_{\pm 0.13}$ & $59.63_{\pm 0.17}$& $85.31_{\pm 0.14}$ & $86.89_{\pm 0.21}$ & $54.15_{\pm0.19}$ & $80.70_{\pm 0.17}$& $35.97_{\pm 0.27}$ & $75.98_{\pm 0.14}$ & $86.56_{\pm 0.11}$ & $85.75_{\pm 0.17}$\\
FedAdam & $79.03_{\pm 0.27}$ & $59.78_{\pm 0.23}$& $85.09_{\pm 0.27}$ & $87.41_{\pm 0.22}$ & $55.21_{\pm 0.25}$ & $80.99_{\pm 0.27}$ & $38.69_{\pm 0.31}$ & $77.10_{\pm 0.19}$& $86.53_{\pm 0.14}$& $87.09_{\pm 0.17}$  \\
SCAFFOLD& $80.21_{\pm 0.19}$ & $60.39_{\pm 0.11}$& $85.99_{\pm 0.13}$ & $87.27_{\pm 0.21}$& $55.86_{\pm 0.24}$ & $81.17_{\pm 0.11}$ & $38.61_{\pm 0.19}$ & $76.57_{\pm 0.11}$ & $88.26_{\pm0.14}$ & $86.87_{\pm 0.13}$\\
FedCM& $80.06_{\pm 0.17}$ & $59.56_{\pm 0.14}$& $85.20_{\pm 0.19}$ & $86.69_{\pm 0.17}$& $55.95_{\pm 0.21}$ & $81.84_{\pm 0.11}$ & $37.89_{\pm 0.25}$ & $77.84_{\pm 0.17}$ & $87.33_{\pm0.13}$ & $85.98_{\pm 0.14}$\\
FedDyn& $77.01_{\pm 0.17}$ & $56.24_{\pm 0.22}$& $83.98_{\pm 0.24}$ & $87.31_{\pm 0.16}$& $52.48_{\pm 0.19}$ & $81.52_{\pm 0.14}$ & $33.10_{\pm 0.29}$ & $76.16_{\pm 0.24}$ & $85.47_{\pm0.13}$ & $86.22_{\pm 0.09}$\\
FedSAM & $79.22_{\pm 0.14}$ & $60.18_{\pm 0.22}$& $86.06_{\pm 0.09}$ & $86.94_{\pm 0.11}$& $55.23_{\pm 0.20}$ & $81.76_{\pm 0.13}$ & $38.90_{\pm 0.26}$ & $77.37_{\pm 0.14}$ & $87.33_{\pm0.11}$ & $86.04_{\pm 0.16}$\\
MoFedSAM & $79.81_{\pm 0.12}$ & $\underline{60.62}_{\pm 0.13}$& $\underline{86.46}_{\pm 0.06}$ & $\underline{87.70}_{\pm 0.17}$& $\underline{56.37}_{\pm 0.19}$ & $\underline{82.28}_{\pm 0.10}$ & $\underline{40.83}_{\pm 0.21}$ & $\underline{77.94}_{\pm 0.17}$ & $87.18_{\pm0.13}$ & $\underline{87.91}_{\pm 0.11}$\\
FedGAMMA& $\underline{80.51}_{\pm 0.11}$ & $60.59_{\pm 0.14}$& $86.35_{\pm 0.17}$ & $87.68_{\pm 0.13}$& $55.38_{\pm 0.20}$ & $81.83_{\pm 0.11}$ & $40.19_{\pm 0.23}$ & $77.30_{\pm 0.14}$ & $\underline{88.83}_{\pm0.09}$ & $87.04_{\pm 0.12}$\\
FedSMOO& $78.70_{\pm 0.21}$ & $57.11_{\pm 0.19}$& $85.43_{\pm 0.11}$ & $87.22_{\pm 0.16}$& $53.44_{\pm 0.31}$ & $81.96_{\pm 0.17}$ & $36.20_{\pm 0.29}$ & $76.94_{\pm 0.11}$ & $86.07_{\pm0.10}$ & $86.65_{\pm 0.09}$\\
FedLESAM& ${79.55}_{\pm 0.19}$ & ${60.57}_{\pm 0.16}$& ${86.49}_{\pm 0.21}$ & ${87.30}_{\pm 0.14}$& ${55.47}_{\pm 0.17}$& ${82.04}_{\pm 0.16}$ & ${39.86}_{\pm 0.12}$ & ${77.42}_{\pm 0.20}$ & ${87.63}_{\pm 0.19}$ & ${86.94}_{\pm0.11}$\\
FedLESAM-D& ${78.85}_{\pm 0.15}$ & ${57.34}_{\pm 0.12}$& ${85.62}_{\pm 0.11}$ & ${86.99}_{\pm 0.10}$& ${54.75}_{\pm 0.18}$ & ${82.24}_{\pm 0.16}$ & ${37.98}_{\pm 0.27}$ & ${77.54}_{\pm 0.22}$ & ${87.12}_{\pm0.19}$ & ${87.54}_{\pm 0.17}$\\
FedLESAM-S& ${\color{red}\textbf{80.73}_{\pm 0.14}}$ & ${\color{red}\textbf{62.13}_{\pm 0.17}}$& ${\color{red}\textbf{87.42}_{\pm 0.19}}$ & ${\color{red}\textbf{87.79}_{\pm 0.11}}$& ${\color{red}\textbf{57.03}_{\pm 0.14}}$ & ${\color{red}\textbf{82.49}_{\pm 0.13}}$ & ${\color{red}\textbf{42.25}_{\pm 0.22}}$ & ${\color{red}\textbf{78.95}_{\pm 0.17}}$ & ${\color{red}\textbf{88.93}_{\pm0.09}}$ & ${\color{red}\textbf{88.74}_{\pm 0.09}}$\\
\bottomrule[2pt]
\end{tabular}
}
\vspace{-6pt}
\end{table*}

\section{Experiments}
\label{sec:exp}
This section introduces some experimental setups including baselines, datasets, splits, and experimental details. Then we show the main results on benchmark datasets followed by extensive further analysis such as ablation and visualization.
\subsection{Experimental Setups}
\textbf{Baselines.} We compare our FedLESAM with FedAvg~\cite{fedavg} and existing federated SAM methods for sloving data heterogeneity including FedSAM~\cite{fedsameccv,fedsamicml}, MoFedSAM~\cite{fedsamicml}, FedGAMMA~\cite{fedgamma}, and FedSMOO~\cite{fedsmoo}. 
We also compare our method with classical federated optimization methods including Scaffold~\cite{scaffold}, FedDyn~\cite{fedyn}, FedAdam~\cite{fedadam}, and FedCM~\cite{fedcm}. 
Since FedGAMMA and FedSMOO respectively draw spirits from Scaffold and FedDyn, for a fair comparison and better minimizing global sharpness, we design two variants named FedLESAM-S and FedLESAM-D based on Scaffold and FedDyn. 
We introduce them in detail in Appendix~\ref{app:variants} and show the ablation in Sec~\ref{sec:further}.

\textbf{Dataset and Splits.} We adopt four popular federated benchmark datasets: CIFAR10/100~\cite{cifar10}, OfficeHome~\cite{officehome} and DomainNet~\cite{domainnet}. 
For CIFAR10/100, we follow \citet{split}, \citet{fedgamma}, and \citet{sun1,sun2,fedsmoo} and use Dirichlet and Pathological splits to simulate Non-IID. 
For OfficeHome and DomainNet, we adopt leave-one-domain-out strategy that selects one domain for test and all other domains for training. 
To simulate \emph{straggler} situations and large scale of clients, we divide CIFAR10/100 into 100 clients with an active ratio of 10\% and 200 clients with an active ratio of 5\%. Each domain in OfficeHome and DomainNet is divided into 1 client with an active ratio of 100\% and 10 clients with an active ratio of 20\%. See Appendix~\ref{app:data} for more details.

\textbf{Experimental Details.}
For a fair comparison on CIFAR10/100, we follow all settings in FedSMOO. 
Backbone is ResNet-18~\cite{resnet} with Group Normalization~\cite{groupnorm} and SGD, total rounds $T=800$, local learning rate $\eta_\mathrm{l}=0.1$, global learning rate $\eta_\mathrm{g}=1$ expect for
FedAdam which adopts $0.1$, and perturbation magnitude $\rho=0.1$ expect for
FedSAM and FedLESAM which adopts $0.01$.
On the CIFAR10, batchsize is $50$, and the local epoch is $5$. 
On the CIFAR100, batchsize equals to $20$, and the local epoch equal to $2$. 
For OfficeHome and DomainNet, we use the pre-trained model ViT-B/32~\cite{vit,dg} as the backbone. 
The optimizer is SGD with a local learning rate 0.001 and a global learning rate $1$ expect for
FedAdam which adopts $0.1$. 
See Appendix~\ref{app:implement} for more details.

\begin{figure*}[t]
    \centering
    \subfigure[FedAvg]{
    \centering
    \label{fig:surface_1}
    \includegraphics[width=0.18\textwidth]{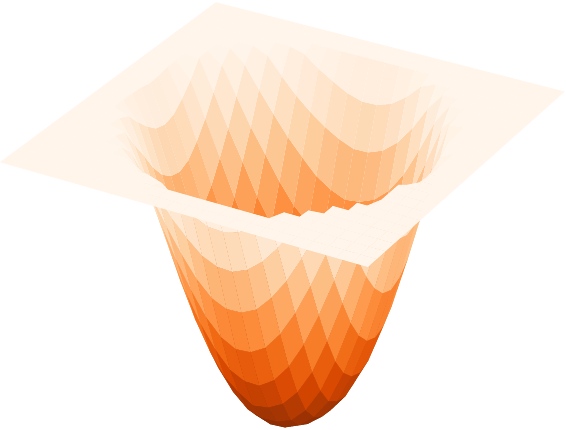}}
    %\hspace{2mm}
    \subfigure[FedSAM]{
    \centering
    \label{fig:surface_2}
    \includegraphics[width=0.18\textwidth]{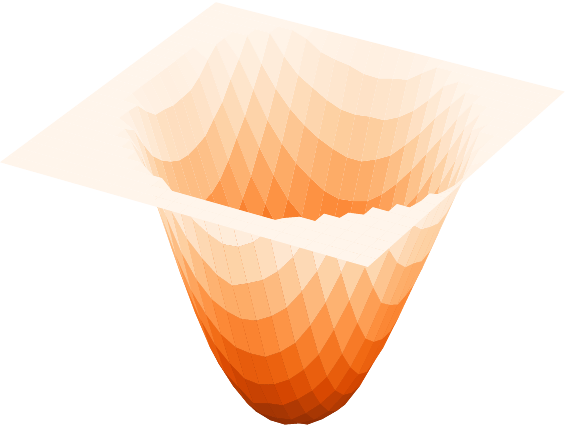}}
    \subfigure[FedGAMMA]{
    \centering
    \label{fig:surface_3}
    \includegraphics[width=0.18\textwidth]{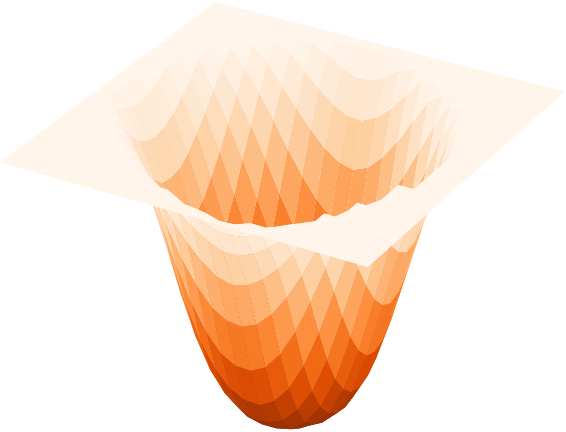}}
    \centering
    \subfigure[FedSMOO]{
    \centering
    \label{fig:surface_4}
    \includegraphics[width=0.18\textwidth]{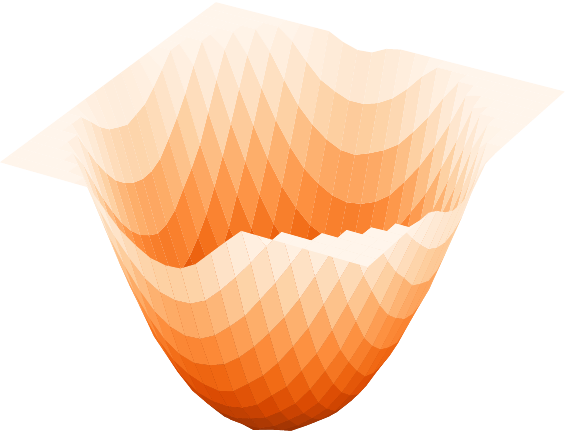}}
    %\hspace{2mm}
    \subfigure[FedLESAM-D]{
    \centering
    \label{fig:surface_5}
    \includegraphics[width=0.18\textwidth]{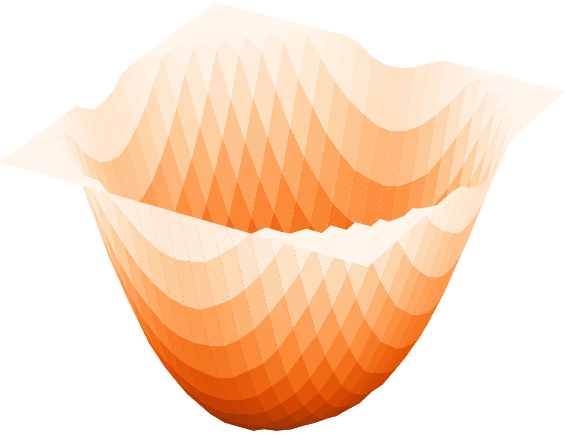}}
    %\vspace{-5pt}
    \caption{Visualization of the global loss surface on CIFAR10 under Dirichlet distribution with coefficient 0.1 of FedAvg, FedSAM, FedGAMMA, FedSMOO and our FedLESAM-D. We divide the dataset into 100 clients and in each round 10\% clients are active.}\label{fig:surface}
    \vspace{-8pt}
\end{figure*}

\begin{figure*}[ht!]
\centering  %图片全局居中
\subfigure{
% \label{fig:svhng}
\includegraphics[width=0.24\textwidth]{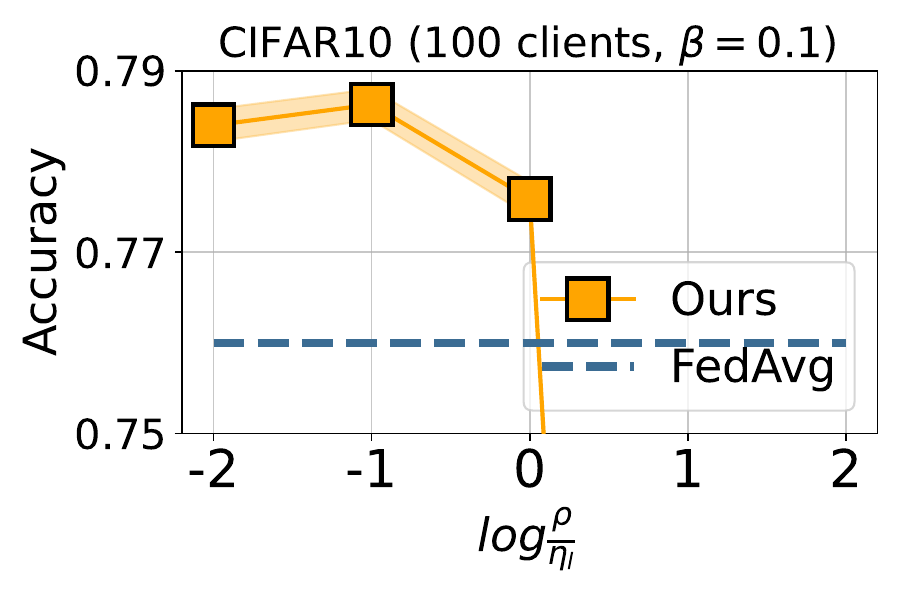}}
\subfigure{
% \label{fig:cifar10g}
\includegraphics[width=0.24\textwidth]{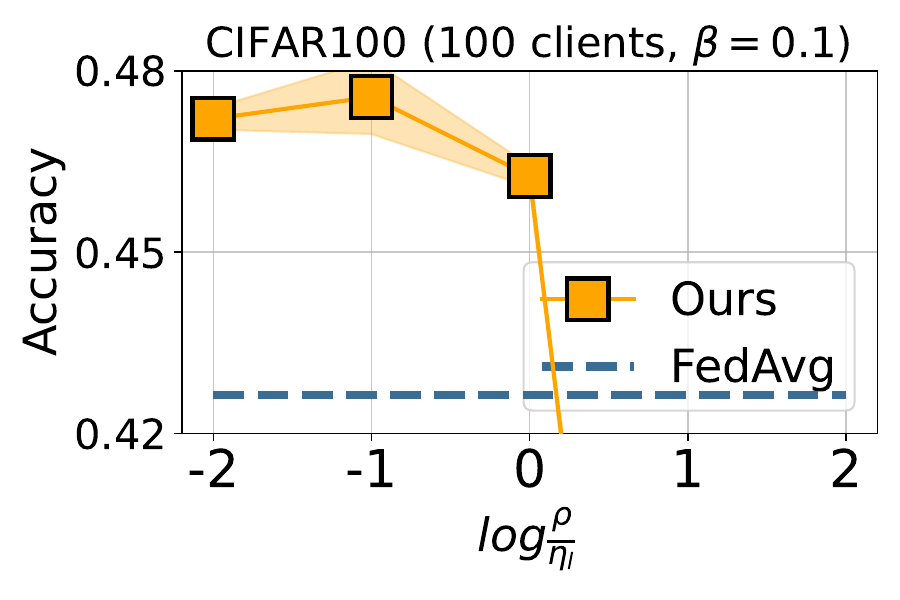}}
\subfigure{
% \label{fig:cifar100g}
\includegraphics[width=0.24\textwidth]{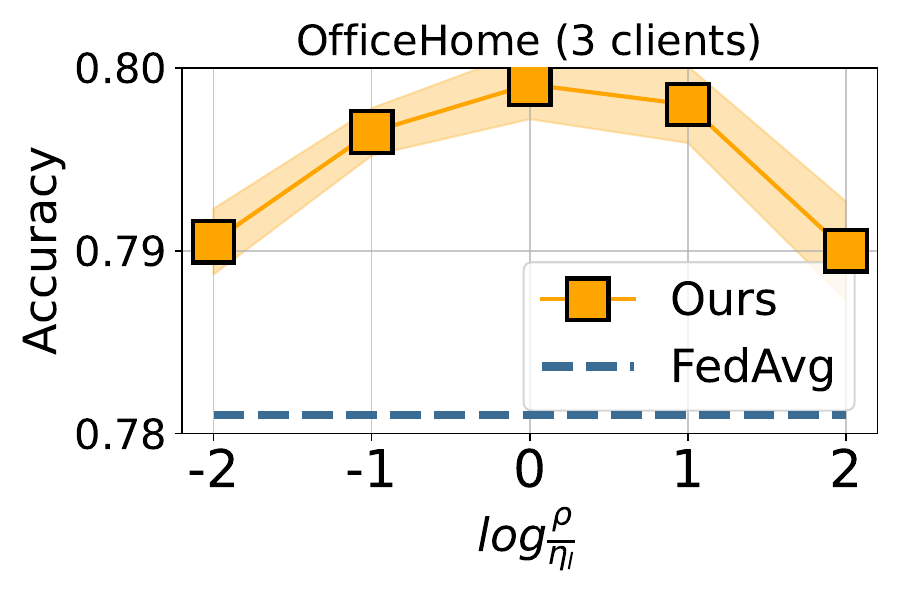}}
\subfigure{
% \label{fig:isicg}
\includegraphics[width=0.24\textwidth]{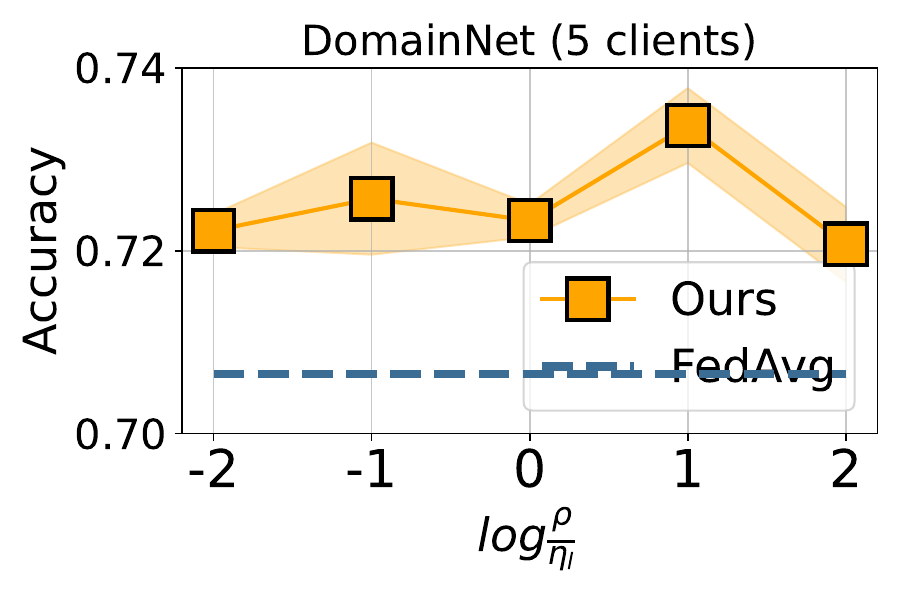}}
\vspace{-8pt}
\caption{Ablation study on $\log\frac{\rho}{\eta_\mathrm{l}}$, where $\eta_\mathrm{l}$ is local learning rate. From left to right, we show the test accuracy on CIFAR10 and CIFAR100~($\eta_\mathrm{l}=0.1$) and the averaged test accuracy of all target domains on OfficeHome and DomainNet~($\eta_\mathrm{l}=0.001$) with different $\rho$.}
\vspace{-15pt}
\label{fig:ablation}
\end{figure*}

\subsection{Main Results}
\textbf{General Performance.} As shown in Table~\ref{tb:cifar}, our method performs the best or second best in the most cases on CIFAR10/100. 
Specifically, in the upper part of Table~\ref{tb:cifar} on CIFAR100, FedLESAM-S outperforms all baselines and achieves averaged improvements of 6\% to FedAvg and 1\% to the best baseline. 
As shown in Table~\ref{tb:domain}, we conduct experiments on OfficeHome and DomainNet under leave-one-domain-out strategy. 
It can be seen that, FedDyn and FedSMOO meet the overfitting problem on unseen domain while our method outperforms all baselines on all settings. 
Especially in the target domain "Painting" shown in the upper part of Table~\ref{tb:domain}, our method achieves improvements of 2.39\% to the best baseline and 7.22\% to FedAvg. 

\textbf{Heterogeneity, Scalability, and Straggler.}
To verify the performance under different levels of data heterogeneity, straggler situations and the scalability to the number of clients, we adopt multiple split strategies. 
As shown in Table~\ref{tb:cifar}, we split CIFAR10/100 into 100 clients and 10\% of them are active at each round in the upper part while 200 and 5\% in the lower part under different coefficient values of Dirichlet and Pathological strategies. 
For DomainNet and OfficeHome, we adopt leave-one-domain-out strategy and each training domain is divided into 1 client and 100\% of them are active at each round in the upper part while 10 and 20\% in the lower part shown in Table~\ref{tb:domain}. 
It can be seen that, compared under all settings, our FedLESAM-S performs well with comprehensive improvement to all baselines. 

\begin{table}[t!]
\centering
\small
\renewcommand\arraystretch{0.98}
\caption{\rm{Ablation study of variants on the averaged test accuracy on the four datasets. FedLESAM with A, S, and D respectively represent variants based on FedAvg, Scaffold, and FedDyn.}}
\vspace{4pt}
    \setlength{\tabcolsep}{1.5pt}
    \scalebox{0.97}{
    \begin{tabular}{l|lllll}
        \toprule[2pt]
        Method&  CIFAR10 & CIFAR100& OfficeHome&DomainNet\\
        \midrule
        FedAvg&$75.96$ &$41.69$&$77.99$&$70.25$ \\
        FedLESAM&$77.64_{1.68\%\uparrow}$ &$43.52_{1.83\uparrow}$&$78.90_{0.91\uparrow}$&$71.98_{1.73\uparrow}$ \\
        Scaffold&$78.80$&$44.21$&$78.93$&$71.62$\\
        FedLESAM-S&$82.63_{3.83\uparrow}$&$\textbf{48.00}_{3.79\uparrow}$&$\textbf{79.80}_{0.87\uparrow}$&$\textbf{73.37}_{1.75\uparrow}$ \\
        FedDyn&$79.60$ &$43.11$&$76.56$&$69.66$\\
        FedLESAM-D&$\textbf{82.65}_{3.05\uparrow}$&$46.78_{3.67\uparrow}$&$77.53_{0.97\uparrow}$&$71.30_{1.64\uparrow}$\\
        \bottomrule[2pt]
    \end{tabular}}
    \label{tab:variants}
    \vspace{-12pt}
\end{table}

\subsection{Further Analysis} \label{sec:further}
\textbf{Ablation of $\rho$.} Since perturbation magnitude $\rho$ critically influences the convergence and performance of SAM-based algorithms, here we tune $\rho$ on the four datasets of FedLESAM-S and compare the averaged test accuracy to FedAvg. 
As shown in Figure~\ref{fig:ablation}, test accuracy initially increases with $\rho$, benefiting from minimizing global sharpness, but then sharply declines as larger perturbations hinder convergence. 
Notably, our FedLESAM-S outperforms FedAvg across a broad range of $\frac{\rho}{\eta_\mathrm{l}}$, especially wider than a range from 10e-2 to 10e2 under OfficeHome and DomainNet. 
Empirically, we recommend setting $\rho$ to approximately 0.1 times the local learning rate $\eta_\mathrm{l}$ when starting with a randomly initialized model, as depicted in the left two panels of Figure~\ref{fig:ablation}, to prevent model breakdown. 
While for pre-trained model as shown in the right two panels of Figure~\ref{fig:ablation}, $\rho$ can be set larger and about 10 times of $\eta_\mathrm{l}$ to better minimize sharpness.

\textbf{Ablation on Variants.} We design FedLESAM under FedAvg and two enhanced methods under Scaffold and FedDyn named FedLESAM-S and FedLESAM-D, respectively. 
Therefore, here we show the averaged performance on CIFAR10/100, OfficeHome and DomainNet of all variants.
As shown in Table~\ref{tab:variants}, all variants achieve extensive improvement to their base methods, especially a notable 3.83\% improvement on CIFAR10 of FedLESAM-S. 
Generally speaking, FedLESAM-S performs the best.

\textbf{Computation and Communication.} Computational time and communication bottleneck are major concerns in FL. Therefore, as shown in Table~\ref{tab:commu}, we compare the total communication rounds, communication costs and computational times of all clients to achieve the target 68\% and 74\% test accuracy of FedAvg, SAM-based algorithms and our two variants. As can be seen that our method achieves competing communication efficiency and slightly smaller communication costs, and greatly reduces the computation. 

\begin{table}[t!]
\centering
\small
\renewcommand\arraystretch{0.95}
\caption{\rm{Total communication rounds, computational time~(minutes) and communication costs~(gigabytes of parameters) to achieve 68\% and 74\% test accuracy under Dirichlet distribution with coefficient of 0.1, 100 clients and 10\% active ratio on CIFAR10 of FedAvg, SAM-based methods and our two variants.}}
\vspace{4pt}
    \setlength{\tabcolsep}{1pt}
    \scalebox{0.95}{
    \begin{tabular}{l|cc|cc|cc}
        \toprule[2pt]
        Method& \multicolumn{2}{c|}{Commu. Round}  & \multicolumn{2}{c|}{Commu. Cost}& \multicolumn{2}{c}{Compu. Time}\\
        \cmidrule{1-7} \#Target Acc.& 68\% & 74\% & 68\% & 74\%& 68\% & 74\%\\
        \midrule
        FedAvg&259 (1x)&786 (1x) &56 (1x)&170 (1x)&28 (1x)&88 (1x)\\
        FedSAM&1.02x&0.83x&1.02x&0.83x&1.96x&1.56x\\
        MoFedSAM&\underline{0.47x}&0.51x&\underline{0.94x}&1.02x&1.08x&1.16x\\
        FedGAMMA&0.76x&0.44x&1.51x&0.89x&1.70x&0.96x\\
        FedSMOO&0.51x&\textbf{0.29x}&1.02x&\textbf{0.58x}&1.15x&0.63x\\
        \midrule
        FedLESAM-S&0.57x&0.38x&1.14x&0.77x&\textbf{0.81x}&\underline{0.53x}\\
        FedLESAM-D&\textbf{0.46x}&\underline{0.31x}&\textbf{0.92x}&\underline{0.60x}&\underline{0.83x}&\textbf{0.30x}\\
        \bottomrule[2pt]
    \end{tabular}}
    \label{tab:commu}
    \vspace{-10pt}
\end{table}
\textbf{Visualization of Global Loss Surface.} As shown in Figure~\ref{fig:surface}, we conduct experiment on CIFAR10 and visualize the global loss surface. Compared to FedAvg, FedSAM and FedGAMMA can not achieve desirable flatness. FedSMOO achieves much flatter loss landscape while our FedLESAM-D further minimizes the global sharpness.

\section{Conclusion}
In this work, we rethink the sharpness-aware minimization~(SAM) in federated learning~(FL) and study the discrepancy between minimizing local and global sharpness under heterogeneous data. To align the efficacy of SAM in FL with centralized training and reduce the computational overheads, we propose a novel and efficient method named FedLESAM and design two effective variants. FedLESAM locally estimates the global perturbation in clients as the difference between the global models received in the last active and current rounds. Theoretically, we provide the convergence guarantee of FedLESAM and prove a slightly tighter bound than its original FedSAM. Empirically, we conducted extensive experiments on four benchmark datasets under three data splits to show the superior performance and efficiency.

\section*{Acknowledgement}
Ziqing Fan, Jiangchao Yao, Ya Zhang and Yanfeng Wang are supported by the National Key R$\&$D Program of China (No. 2022ZD0160702),  STCSM (No. 22511106101, No. 22511105700, No. 21DZ1100100), 111 plan (No. BP0719010) and National Natural Science Foundation of China (No. 62306178). Ziqing Fan is partially supported by Wu Wen Jun Honorary Doctoral Scholarship, AI Institute, Shanghai Jiao Tong University. Masashi Sugiyama is supported by Institute for AI and Beyond, UTokyo and by a grant from Apple, Inc. Any views, opinions, findings, and conclusions or recommendations expressed in this material are those of the authors and should not be interpreted as reflecting the views, policies or position, either expressed or implied, of Apple Inc.

\section*{Impact Statement}
The approach proposed in this paper is computationally economical, adding no extra overhead, and demonstrates superior performance. This advancement has significant implications for federated applications in sensitive fields like medical diagnosis and autonomous driving, where data privacy is paramount and computation resources are at a premium. To date, our analysis has not revealed any negative impacts of this method.

% \bibliography{main.bib}
% \bibliographystyle{icml2024}

\appendix
\onecolumn
\newpage

\section{Related Works}\label{app:related}
\textbf{Federated Learning.} Federated learning~(FL) has drawn the considerable attention due to the increasing concerns in collaboration learning~\cite{privacy1,privacy2,multiagent}. 
Its base framework, FedAvg~\cite{fedavg}, allows clients keeping their private data in the local and cooperatively train a global model, however suffering from the data distribution shifts among clients. 
Recently, many optimization based methods are proposed to solve the problem. 
FedAdam~\cite{fedadam} designs an efficient adaptive optimizer in the server to improve the performance. 
Scaffold~\cite{scaffold} designs a variance reduction approach to achieve a stable and fast local update. 
FedDyn~\cite{fedyn} introduces a dynamic regularizer for each client at each round to align global and local objectives. 
FedCM~\cite{fedcm} maintains the consistency of local updates with a momentum term. 
Our method is also optimization based and orthogonal to these methods. 
In the experiments for a fair comparison to FedGAMMA~\cite{fedgamma} and FedSMOO~\cite{fedsmoo} and better minimize global sharpness, we design two effective variants based on the frameworks Scaffold and FedDyn, named FedLESAM-S and FedLESAM-D, respectively.

\textbf{Sharpness-aware Minimization.} Many studies~\citep{first_flatness_nips1994,loss_landscape_nips_2018,sharp_minima_icml_2017} have demonstrated that a flat minimum tends to exhibit superior generalization performance in deep learning models.
To minimize both the sharpness metric~\citep{keskar2017on} and the training loss, \citet{sam} proposes the sharpness-aware minimization~(SAM) and many works are proposed to improve it from the views of generalizability~\citep{ASAM,GSAM,GAM,PGN,add1,add2,vasso,msharpness} and the efficiency~\citep{ESAM, SAF, SparseSAM}. 
Specifically, SAM's adversary captures the sharpness of only a specific minibatch of data, and VaSSO~\cite{vasso} aims to address this "friendly adversary" problem. Our FedLESAM may also help solve this problem to some extent by treating the global update (accumulated average gradients from many batches of data) as the perturbation.
Furthermore, m-sharpness~\cite{msharpness} can be considered closely related to federated sharpness minimization. m-sharpness quantifies the sharpness across batches of m training points, and the corresponding optimization method, mSAM, averages the updates generated by adversarial perturbations across multiple disjoint shards of a mini-batch. In federated learning settings, if client models align with global model in the local training, the optimization of FedSAM is similar to mSAM. Additionally, FedLESAM calculates perturbations based on $w_i^{old}-w^t$, which are the accumulated adversaries from many clients' batch data, also reflecting the principle of minimizing m-sharpness. 
Our method is an efficient federated SAM algorithm and orthogonal to existing SAM methods. 
Therefore in the paper, we use the original SAM~\cite{sam} as the base algorithm. We leave the combination of our FedLESAM with other enhanced SAM
algorithms to the future work.

\textbf{Sharpness-aware Minimization in FL.} To utilize the generalization and sharpness minimizing ability of SAM in federated learning, \citet{fedsamicml} and \citet{fedsameccv} proposed FedSAM by adding sharpness optimization into local training.
\citet{fedsamicml} proposed a momentum variant of FedSAM named MoFedSAM. FedGAMMA~\cite{fedgamma} learned from Scaffold and introduced variance reduction to FedSAM.
With the similar spirit of FedDyn~\cite{fedyn}, FedSMOO~\cite{fedsmoo} adopts a dynamic regularizer to guarantee the local optima towards the global objective and add a correction to local perturbations to search for the consistent flat minima. 
We summarize them in the Table~\ref{tab:related} from the views of how to calculate perturbation in local clients, target of the local sharpness optimization~(target on the local or global sharpness) and their base federated algorithms. 
As can be seen that, FedSAM, MoFedSAM and FedGAMMA calculate local perturbations and optimize the sharpness on the client data, which might not direct global model to a global flat minimum. 
Although FedSMOO notices the conflicts and add a correction, which still need to calculate the local perturbations. 
All above algorithms introduce extra computational burden on the local, which might increase the expenses of clients in the federation. 
Therefore, we propose an efficient algorithm that Locally-Estimating Global perturbation for SAM~(FedLESAM), that can both optimize global sharpness and reduce the computation.
\setcounter{theorem}{0}
\setcounter{assumption}{0}
\section{Implementation of Theoretical Analysis}\label{app:proof}
Before start our proof for Theorem~\ref{thm:naive} and Theorem~\ref{thm:gefedsam}, we first pre-define some notations, assumptions, and key lemmas used in the proof. 

\subsection{Notations and Assumptions}
\begin{assumption}[$L$-smooth and bounded variance of unit stochastic gradients] 
$F_1, \cdots, F_N$ are all $L$-smooth:
\begin{equation}
    \left\|\nabla F_i(u)-\nabla F_i(v)\right\| \leq L\|u-v\| \nonumber,
\end{equation}
and the variance of unit stochastic gradients is bounded:
\begin{equation}
\mathbb{E}\left\|\frac{\nabla F_i\left(u, \xi_i\right)}{\left\|\nabla F_i\left(u, \xi_i\right)\right\|}-\frac{\nabla F_i(u)}{\left\|\nabla F_i(u)\right\|}\right\|^2 \leq \sigma_\mathrm{l}^2. \nonumber
\end{equation}
\label{asm:smooth_var2}
\end{assumption}

\begin{assumption}[Bounded heterogeneity] 
The gradient difference between $F(u)$ and $F_i(u)$ is bounded:
\begin{equation}
   \left\|\nabla F_i\left(u\right)-\nabla F\left(u\right)\right\|^2 \leq \sigma_\mathrm{g}^2 \nonumber
\end{equation}
\label{asm:grad_differ2}
\vspace{-14pt}
\end{assumption}

\begin{assumption}[Bounded unit variance] 
Variance of unit averaged stochastic gradients is bounded:
\begin{equation}
\mathbb{E}\left\|\frac{\sum_{i=1}^N \nabla F_i\left(u, \xi_i\right)}{\left\|\sum_{i=1}^N \nabla F_i\left(u, \xi_i\right)\right\|}-\frac{\sum_{i=1}^N \nabla F_i\left(u\right)}{\left\|\sum_{i=1}^N \nabla F_i\left(u\right)\right\|}\right\|^2\leq \sigma_\mathrm{l}'^2. \nonumber
\end{equation}
\label{asm:unit_variance2}
\vspace{-14pt}
\end{assumption}

\begin{assumption}[Bounded unit difference] 
The variance of unit averaged gradient difference between $F(u)$ and $\sum_{i=1}^NF_i(u)$ is bounded:
\begin{equation}
\frac{\sum_{i=1}^N \nabla F_i\left(u\right)}{\left\|\sum_{i=1}^N \nabla F_i\left(u\right)\right\|}-\frac{\nabla F\left(u\right)}{\| \nabla F\left(u\right)\|}\leq \sigma_\mathrm{g}'^2. \nonumber
\end{equation}
\label{asm:unit_differ2}
\vspace{-14pt}
\end{assumption}

\begin{assumption}[$L_\mathrm{g}$-smooth] 
Global objective $F$ is $L_\mathrm{g}$-smooth:
\begin{equation}
    \left\|\nabla F(u)-\nabla F(v)\right\| \leq L_\mathrm{g}\|u-v\| \nonumber,
\end{equation}
\label{asm:smooth2}
\end{assumption}
We use $i, k, t$ to denote the client id, the number of iterations in a round and the number of communication rounds, respectively. For example, $w_{i,k}^t$ means model weights of $i$-th client in $k$-th iterations at t-th rounds. Given local loss function $F_i$, global function $F$, $N$ clients and $E$ pre-defined local iterations at round $t$, the update of local models in FedSAM and FedLESAM can be defined as follows:
$$\begin{aligned} & \tilde{w}_{i, k}^t=w_{i, k-1}^t+\rho \delta_{i,k}^t \\ & w_{i, k}^t=w_{i, k-1}^t-\eta_t \frac{\nabla F_i(\tilde{w}_{i, k-1}^t,\xi_{i,k}^t)}{\nabla F_i(\tilde{w}_{i, k-1}^t,\xi_{i,k}^t)},\end{aligned}$$
where $\xi_{i,k}^t$ is randomly sampled in the local dataset, $\rho$ is the pre-defined perturbation magnitude and $\eta_\mathrm{l}$ is local learning rate. After E steps, the local clients submit their trained local models to the server, and in the sever, all local models are aggregated to a new global model as following:
$$w^{t+1}=w^{t}-\eta_\mathrm{g} \frac{1}{N}\sum_{i=0}^{N-1}(w_{i,E-1}^t-w^{t}),$$
where $\eta_\mathrm{g}$ is the global learning rate. The difference of FedSAM and FedLESAM is the definition of perturbation. In FedSAM, the perturbation is calculated as:
$$
\delta_{i,k}^t=\frac{\nabla F_i(w_{i, k-1}^t,\xi_{i,k}^t)}{\|\nabla F_i(w_{i, k-1}^t,\xi_{i,k}^t)\|},
$$
where $\xi_{i,k}^t$ is randomly sampled in the local dataset. However, the perturbation in our FedLESAM under full participation is defined as follows:
$$
\delta_{i,k}^t=\frac{w^{t-1}-w^t}{\|w^{t-1}-w^t\|}.
$$

Then, we will introduce the some basic assumptions on loss functions~$F_1, F_2,\cdots, F_N$ of all clients and their gradient functions~$\nabla F_1, \nabla F_2,\cdots, \nabla F_N$, which are the same as FedSAM~\cite{fedsamicml}. In Assumption~\ref{asm:smooth_var2} and Assumption~\ref{asm:grad_differ2}, we characterize the smoothness, the bound on the variance of unit stochastic gradients, and the bound on the  gradient difference between local and global objectives induced by data heterogeneity. In Assumption~\ref{asm:smooth2}, we assume the smoothness of the global objective $F$ for proving reasonableness of the perturbation estimation under a naive case.

\subsection{Key Lemmas}
Here we introduce some lemmas proofed by previous research~\cite{fedsamicml} and use them as intermediate results in our proof. For the convenience of the reading, we provide the proof of some lemmas and update the results of our FedLESAM in Lemma~\ref{lem:2}. 
\begin{lemma}[Intermediate results]\label{lem:1} 
Let Assumption~\ref{asm:smooth_var2} hold, $\left\langle\nabla F\left(\tilde{w}^t\right), \mathbb{E}\left[\frac{1}{N}\sum_{i=0}^N (w_{i, E-1}^t-w^t)+\eta_\mathrm{l} E \nabla F\left(\tilde{w}^t\right)\right]\right\rangle$ can be bounded as:
$$\begin{aligned} & 
\left\langle\nabla F\left(\tilde{w}^t\right), \mathbb{E}\left[\frac{1}{N}\sum_{i=0}^N (w_{i, E-1}^t-w^t)+\eta_\mathrm{l} E \nabla F\left(\tilde{w}^t\right)\right]\right\rangle \leq  \frac{\eta_\mathrm{l} E}{2} \| \nabla F\left(\tilde{w}^t\right)\|^2+E \eta_\mathrm{l} L^2 \frac{1}{N} \sum_{i=0}^N \mathbb{E}\left[\left\|w^t_{i, k}-w^t\right\|^2\right] \\
& +E \eta_\mathrm{l} L^2 \frac{1}{N} \sum_{i=0}^{N} \mathbb{E}\left[\left\|\delta_{i, k}^t-\delta_{i,0}^t\right\|^2\right]-\frac{\eta_\mathrm{l}}{2 E N^2} \mathbb{E} \| \sum_{i, k} \nabla F_i\left(\tilde{w}_{i, k}\right) \|^2 \end{aligned}$$
\end{lemma}

\begin{proof}
$$\begin{aligned} & \left\langle\nabla F\left(\tilde{w}^t\right), \mathbb{E}_t\left[\frac{1}{N}\sum_{i=0}^N (w_{i, E-1}^t-w^t)+\eta_\mathrm{l} E \nabla F\left(\tilde{w}^t\right)\right]\right\rangle \\ 
& \stackrel{(\text { a) }}{=} \frac{\eta_\mathrm{l} E}{2} \| \nabla F\left(\tilde{w}^t\right)\|^2+\frac{\eta_\mathrm{l}}{2 K N^2} \mathbb{E}_t \| \sum_{i, E} \nabla F_i\left(\tilde{w}_{i, k}^t\right)-\nabla F_i\left(\tilde{w}^t\right)\left\|^2-\frac{\eta_\mathrm{l}}{2 E N^2} \mathbb{E}_t\right\| \sum_{i, k} \nabla f_i\left(\tilde{w}_{i, k}^t\right) \|^2 \\ 
& \stackrel{(\text { b) }}{\leq} \frac{\eta_\mathrm{l} E}{2} \| \nabla f\left(\tilde{w}^t\right)\|^2+\frac{\eta_\mathrm{l}}{2 N} \sum_{i, k} \mathbb{E}_t \| \nabla F_i\left(\tilde{w}_{i, E}^t\right)-\nabla F_i\left(\tilde{w}^t\right)\left\|^2-\frac{\eta_\mathrm{l}}{2 E N^2} \mathbb{E}_t\right\| \sum_{i, k} \nabla f_i\left(\tilde{w}_{i, k}^t\right) \|^2 \\ 
& \stackrel{(\text { c) }}{\leq} \frac{\eta_\mathrm{l} K}{2} \| \nabla F\left(\tilde{w}^t\right)\|^2+\frac{\eta_\mathrm{l} \beta^2}{2 N} \sum_{i, k} \mathbb{E}_t\| \tilde{w}_{i, k}^t-\tilde{w}_{i,0}^t \|^2-\frac{\eta_\mathrm{l}}{2 E N^2} \mathbb{E}_t \| \sum_{i, k} \nabla F_i\left(\tilde{w}_{i, k}^t\right) \|^2 \\ 
&\stackrel{(\text { d) }}{\leq} \frac{\eta_\mathrm{l} E}{2} \| \nabla F\left(\tilde{w}^t\right)\|^2+\frac{\eta_\mathrm{l} L^2}{N} \sum_{i, k} \mathbb{E}_t\| w_{i, k}^t-w_{i,0}^t\|^2+\frac{\eta_\mathrm{l} L^2}{N} \sum_{i, k} \mathbb{E}_t\| \delta_{i, k}^t-\delta_{i,0}^t\|^2-\frac{\eta_\mathrm{l}}{2 E N^2} \mathbb{E}_t\| \sum_{i, k} \nabla f_i\left(\tilde{w}_{i, k}^t\right) \|^2, \end{aligned}$$
where $(a)$ are because $\langle a,b \rangle =\frac{1}{2}(\|a\|^2+\|b\|^2-\|a-b\|^2)$ with $a=\sqrt{\eta_\mathrm{l} E}\nabla F(\tilde{w}^t)$ and $b=-\frac{\sqrt{\eta_\mathrm{l}}}{N\sqrt{E}}\sum_{i,k}(\nabla F_i(\tilde{w}_{i,k}^t)-\nabla F_i(\tilde{w}_{i,0}^t))$; $(b)$ and $(d)$ is because, for random variables $x_1,..., x_n$, $\mathbb{E}\left[\left\|x_1+\cdots+x_n\right\|^2\right] \leq n \mathbb{E}\left[\left\|x_1\right\|^2+\cdots+\left\|x_n\right\|^2\right]$; $(c)$ is from Assumption~\ref{asm:smooth_var2}.
\end{proof}

\begin{lemma}[Bounded perturbation difference]\label{lem:2} 
Let Assumption~\ref{asm:smooth_var2} and \ref{asm:grad_differ2} hold, given local perturbations $\delta_{i,k}^t~(k=0, 1,..., E-1)$ at any step and local perturbation $\delta_{i,0}^t$ at the first step, the variance of perturbation difference in FedSAM can be bounded as:
$$\frac{1}{N} \sum_i \mathbb{E}\left[\left\|\delta_{i, k}^t-\delta_{i,0}^t\right\|^2\right] \leq  2K^2 L^2 \eta_\mathrm{l}^2 \rho^2.$$ 
However in our FedLESAM, it is zero since the perturbation is consistent during the local training within a round:
$$
\frac{1}{N} \sum_i \mathbb{E}\left[\left\|\delta_{i, k}^t-\delta_{i,0}^t\right\|^2\right]=0
$$
\end{lemma}

\begin{lemma}[Bounded variance of gradient difference after perturbation]\label{lem:3} 
Let Assumption~\ref{asm:smooth_var2} and \ref{asm:grad_differ2} hold, the variance of gradient difference after perturbation can be bounded as:

$$\left\|\nabla F_i\left(w+\delta_i\right)-\nabla F(w+\delta)\right\|^2 \leq 3 \sigma_g^2+6 L^2 \rho^2.$$

\end{lemma}
\begin{proof}
$$
\begin{aligned}\left\|\nabla f_i(\tilde{w})-\nabla f(\tilde{w})\right\|^2 & =\left\|\nabla F_i\left(w+\delta_i\right)-\nabla F(w+\delta)\right\|^2 \\ & =\left\|\nabla F_i\left(w+\delta_i\right)-\nabla F_i(w)+\nabla F_i(w)-\nabla F(w)+\nabla F(w)-\nabla F(w+\delta)\right\|^2 \\ & \stackrel{\text { (a) }}{\leq} 3\left\|\nabla F_i\left(w+\delta_i\right)-\nabla F_i(w)\right\|^2+3\left\|\nabla F_i(w)-\nabla F(w)\right\|^2+3\|\nabla F(w)-\nabla F(w+\delta)\|^2 \\ & \stackrel{\text { (b) }}{\leq} 3 \sigma_g^2+6 L^2 \rho^2,\end{aligned}
$$
where (a) is because, for random variables $x_1,..., x_n$, $\mathbb{E}\left[\left\|x_1+\cdots+x_n\right\|^2\right] \leq n \mathbb{E}\left[\left\|x_1\right\|^2+\cdots+\left\|x_n\right\|^2\right]$ and b is from Assumption~\ref{asm:smooth_var2} and Assumption~\ref{asm:grad_differ2}.
\end{proof}

\begin{lemma}[Bounded iteration difference]\label{lem:4} 
 Suppose local functions satisfy Assumptions~\ref{asm:smooth_var2}-\ref{asm:grad_differ2}. Then, if learning rate satisfy $\eta_\mathrm{l} \leq \frac{1}{10EL}$, the update difference at any iterations within a round can be bounded as
$$\begin{aligned} \frac{1}{N} \sum_{i \in[N]} & \mathbb{E}\left\|w_{i, k}^t-w^t\right\|^2 {\leq}5 E \eta_\mathrm{l}^2\left(2 L^2 \rho^2 \sigma_l^2+6 E\left(3 \sigma_g^2+6 L^2 \rho^2\right)+6 E\|\nabla F(\tilde{w})\|^2 +12 E L^2 \eta_\mathrm{l}^2 \frac{1}{N} \sum_{\mathbb{E} \|}\left\|\delta_{i, k}^t-\delta_{i,0}^t\right\|^2\right).\end{aligned}$$
\end{lemma}

\begin{lemma}[Bounded update difference]\label{lem:5} 
The squared norm of averaged update difference can be bounded as:
$$\mathbb{E}\left[\left\|\frac{1}{N}\sum_{i=0}^N (w_{i, E-1}^t-w^t)\right\|^2\right] \leq \frac{K \eta_\mathrm{l}^2 L^2 \rho^2}{N} \sigma_l^2+\frac{\eta_\mathrm{l}^2}{N^2}\left[\left\|\sum_{i, k} \nabla f_i\left(\tilde{w}_{i, k}^t\right)\right\|^2\right].$$
\end{lemma}

\begin{lemma}[Descent Lemma]\label{lem:6} 
Let Assumption~\ref{asm:smooth_var2}-\ref{asm:grad_differ2} hold, the loss function at any round satisfies the following relationship:
$$\begin{aligned} & \mathbb{E}\left[F\left(w^{t+1}\right)\right]\leq F\left(\tilde{w}^t\right)-E \eta_\mathrm{g} \eta_\mathrm{l}\left(\frac{1}{2}-30 E^2 L^2 \eta_\mathrm{l}^2\right) \| \nabla F\left(\tilde{w}^t\right)\|^2
+10 E^2 L^4 \eta_\mathrm{l}^3 \rho^2 \sigma_l^2+90 E^3 L^2 \eta_\mathrm{l}^3 \sigma_g^2+180 E^3 L^4 \eta_\mathrm{l}^3 \rho^2\\
&+\frac{EL^2\eta_\mathrm{l}}{N}(60 \eta_\mathrm{l}^2+1)   \sum_{i=0}^{N} \mathbb{E}\left\|\delta_{i, k}^t-\delta_{i,0}^t\right\|^2 +\frac{1}{2N} \eta_\mathrm{g}^2 E \eta_\mathrm{l}^2 L^3 \rho^2 \sigma_l^2.
\end{aligned}$$
\end{lemma}

\begin{proof}
$$\begin{aligned} & \left.\mathbb{E}\left[F\left(w^{t+1}\right)\right]=\mathbb{E}\left[F\left(\tilde{w}^{t+1}\right)\right] \leq F\left(\tilde{w}^t\right)+\mathbb{E}\left\langle\nabla F\left(\tilde{w}^t\right), \tilde{w}^{t+1}-\tilde{w}^t\right]\right\rangle+\frac{L}{2} \mathbb{E}_t\left[\left\|\tilde{w}^{t+1}-\tilde{w}^t\right\|^2\right] \\ 
& \stackrel{(\mathrm{a})}{=} F\left(\tilde{w}^t\right)+\mathbb{E}_t\left\langle\nabla F\left(\tilde{w}^t\right),-\frac{1}{N}\sum_{i=0}^N (w_{i, E-1}^t-w^t)+K \eta_\mathrm{g} \eta_\mathrm{l} \nabla F\left(\tilde{w}^t\right)-E \eta_\mathrm{g} \eta_t \nabla F\left(\tilde{w}^t\right)\right\rangle \\
&+\frac{L}{2} \eta_\mathrm{g}^2 \mathbb{E}\left[\left\|\frac{1}{N}\sum_{i=0}^N (w_{i, E-1}^t-w^t)\right\|^2\right] \\ 
& \stackrel{(b)}{=} F\left(\tilde{w}^t\right)-E \eta_\mathrm{g} \eta_\mathrm{l}\left\|\nabla F\left(\tilde{w}^t\right)\right\|^2+\eta_\mathrm{g}\left\langle\nabla F\left(\tilde{w}^t\right), \mathbb{E}\left[-\frac{1}{N}\sum_{i=0}^N (w_{i, E-1}^t-w^t)-w^t)+E \eta_\mathrm{l} \nabla F\left(\tilde{w}^t\right)\right]\right\rangle\\ 
&+\frac{L}{2} \eta_\mathrm{g}^2 \mathbb{E}\left[\left\|\frac{1}{N}\sum_{i=0}^N (w_{i, E-1}^t-w^t)\right\|^2\right],
\end{aligned}$$
where $(a)$ is from the client update defined in Algorithm~\ref{alg:fedgesam}; $(b)$ is from the unbiased estimators. Combining the results shown in Lemma~\ref{lem:1}, we have:
$$\begin{aligned} & \mathbb{E}\left[F\left(w^{t+1}\right)\right]\leq F\left(\tilde{w}^t\right)-E \eta_\mathrm{g} \eta_\mathrm{l}\left\|\nabla F\left(\tilde{w}^t\right)\right\|^2+\eta_\mathrm{g}\frac{\eta_\mathrm{l} E}{2} \| \nabla f\left(\tilde{w}^t\right)\|^2+E \eta_\mathrm{l}\eta_\mathrm{g} L^2 \frac{1}{N} \sum_{i=0}^N \mathbb{E}\left[\left\|w^t_{i, k}-w^t\right\|^2\right] \\
& +E \eta_\mathrm{g}\eta_\mathrm{l} L^2 \frac{1}{N} \sum_{i=0}^{N} \mathbb{E}\left[\left\|\delta_{i, k}^t-\delta_{i,0}^t\right\|^2\right]-\frac{\eta_\mathrm{g}\eta_\mathrm{l}}{2 E N^2} \mathbb{E} \| \sum_{i, k} \nabla F_i\left(\tilde{w}_{i, k}\right) \|^2 +\frac{L}{2} \eta_\mathrm{g}^2 \mathbb{E}\left[\left\|\frac{1}{N}\sum_{i=0}^N (w_{i, E-1}^t-w^t)\right\|^2\right].
\end{aligned}$$
Combining the results in Lemma~\ref{lem:4}, we have:
$$\begin{aligned} & \mathbb{E}r\left[F\left(w^{t+1}\right)\right]\leq F\left(\tilde{w}^t\right)-E \eta_\mathrm{g} \eta_\mathrm{l}\left\|\nabla F\left(\tilde{w}^t\right)\right\|^2+\eta_\mathrm{g}\frac{\eta_\mathrm{l} E}{2} \| \nabla f\left(\tilde{w}^t\right)\|^2\\
&+E \eta_\mathrm{l}\eta_\mathrm{g} L^2 5 E \eta_\mathrm{l}^2\left(2 L^2 \rho^2 \sigma_l^2+6 E\left(3 \sigma_g^2+6 L^2 \rho^2\right)+6 E\|\nabla f(\tilde{w})\|^2\right)+60 E^2 L^4 \eta_\mathrm{l}^3\eta_\mathrm{g} \frac{1}{N} \sum_{\mathbb{E} \|}\left\|\delta_{i, k}^t-\delta_{i,0}^t\right\|^2 \\
& +E \eta_\mathrm{l}\eta_\mathrm{g} L^2 \frac{1}{N} \sum_{i=0}^{N} \mathbb{E}\left[\left\|\delta_{i, k}^t-\delta_{i,0}^t\right\|^2\right]-\frac{\eta_\mathrm{l}\eta_\mathrm{g}}{2 E N^2} \mathbb{E} \| \sum_{i, k} \nabla F_i\left(\tilde{w}_{i, k}^t\right) \|^2 +\frac{L}{2} \eta_\mathrm{g}^2 \mathbb{E}\left[\left\|\frac{1}{N}\sum_{i=0}^N (w_{i, E-1}^t-w^t)\right\|^2\right].
\end{aligned}$$
Due to the results in Lemma~\ref{lem:5}, it satisfies:
$$\begin{aligned} & \mathbb{E}\left[F\left(w^{t+1}\right)\right]\leq F\left(\tilde{w}^t\right)-E \eta_\mathrm{g} \eta_\mathrm{l}\left\|\nabla F\left(\tilde{w}^t\right)\right\|^2+\eta_\mathrm{g}\frac{\eta_\mathrm{l} E}{2} \| \nabla f\left(\tilde{w}\right)\|^2\\
&+E \eta_\mathrm{g}\eta_\mathrm{l} L^2 5 E \eta_\mathrm{l}^2\left(2 L^2 \rho^2 \sigma_l^2+6 E\left(3 \sigma_g^2+6 L^2 \rho^2\right)+6 E\|\nabla F(\tilde{w})\|^2\right)+60 E^2 L^4 \eta_\mathrm{g}\eta_\mathrm{l}^3 \frac{1}{N} \sum_{\mathbb{E} \|}\left\|\delta_{i, k}-\delta_{i,0}^t\right\|^2 \\
& +E \eta_\mathrm{g}\eta_\mathrm{l} L^2 \frac{1}{N} \sum_{i=0}^{N} \mathbb{E}\left[\left\|\delta_{i, k}^t-\delta_{i,0}^t\right\|^2\right]-\frac{\eta_\mathrm{l}}{2 E N^2} \mathbb{E} \| \sum_{i, k} \nabla F_i\left(\tilde{w}_{i, k}^t\right) \|^2 +\frac{L}{2} \eta_\mathrm{g}^2 \frac{E \eta_\mathrm{l}^2 L^2 \rho^2}{N} \sigma_l^2+\frac{\eta_\mathrm{l}^2}{N^2}\left[\left\|\sum_{i, k} \nabla F_i\left(\tilde{w}_{i, k}^t\right)\right\|^2\right].
\end{aligned}$$
If $\eta_\mathrm{l} \leq \frac{1}{2E}$, we can summarize it as following:
$$\begin{aligned} & \mathbb{E}\left[F\left(w^{t+1}\right)\right]\leq F\left(\tilde{w}^t\right)-E \eta_\mathrm{g} \eta_\mathrm{l}\left(\frac{1}{2}-30 E^2 L^2 \eta_\mathrm{l}^2\right) \| \nabla F\left(\tilde{w}^t\right)\|^2
+10 E^2 L^4 \eta_\mathrm{l}^3\eta_\mathrm{g} \rho^2 \sigma_l^2+90 E^3 L^2 \eta_\mathrm{l}^3 \sigma_g^2\\
&+180 E^3 L^4 \eta_\mathrm{l}^3\eta_\mathrm{g} \rho^2+\frac{EL^2\eta_\mathrm{l}\eta_\mathrm{g}}{N}(60EL^2 \eta_\mathrm{l}^2+1)   \sum_{i=0}^{N} \mathbb{E}\left\|\delta_{i, k}^t-\delta_{i,0}^t\right\|^2 +\frac{1}{2N} \eta_\mathrm{g}^2 E \eta_\mathrm{l}^2 L^3 \rho^2 \sigma_l^2.
\end{aligned}$$
\end{proof}

\subsection{Proof of Theorem~1}
Here we provide the proof of Theorem~\ref{thm:gefedsam}.
\begin{theorem}
    \label{thm:gefedsam2}
    Let Assumption~\ref{asm:smooth_var}-\ref{asm:grad_differ} hold, with an independent $\rho$ under full participation, if choosing $\eta_\mathrm{l}=\frac{1}{\sqrt{T} E L}$ and $\eta_\mathrm{g}=\sqrt{E N}$, the sequence of $\{w^t\}$ generated by FedSAM and FedLESAM in Algorithm~\ref{alg:fedgesam} satisfies:
    \begin{equation}
    \frac{1}{T} \sum_{t=1}^T \mathbb{E}\left[\left\|\nabla F\left(w^{t+1}\right)\right\|\right]\leq \frac{10L(F\left(w^0\right)-F^*)}{C\sqrt{T E N}}+\frac{90L^2\rho^2\sigma_\mathrm{g}^2}{CTE}+\frac{180 L^2\rho^2}{C T}+\Delta+\frac{L^2\sigma_\mathrm{l}^2 \rho^2}{C\sqrt{T E N}}, \nonumber
\end{equation}
    where $C \geq (\frac{1}{2}-30 E^2 L^2 \eta_\mathrm{l}^2) \geq 0$. For FedSAM,  $\Delta=\frac{120L^2\rho^2}{CET^2}+\frac{2L^2\rho^2}{CT}$, while for our FedLESAM, $\Delta=0$.
\end{theorem}

\begin{proof}
Summing results of Lemma~\ref{lem:6}, define $C \geq (\frac{1}{2}-30 E^2 L^2 \eta_\mathrm{l}^2) \geq 0$, we have:
$$\begin{aligned} & \frac{1}{T} \sum_{t=1}^T \mathbb{E}\left[\left\|\nabla F\left(w^{t+1}\right)\right\|^2\right]=\frac{1}{T} \sum_{t=1}^T \mathbb{E}\left[\left\|\nabla F\left(\tilde{w}^{t+1}\right)\right\|^2\right] \\ 
& \leq \frac{F\left(\tilde{w}^t\right)-F\left(\tilde{w}^{t+1}\right)}{C E \eta_\mathrm{g} \eta_\mathrm{l} T}\\ 
& +\frac{1}{C}\left(10 E L^4 \eta_\mathrm{l}^2 \rho^2 \sigma_l^2+90 E^2 L^2 \eta_\mathrm{l}^2 \sigma_g^2+180 E^2 L^4 \eta_\mathrm{l}^2 \rho^2+\frac{L^2}{N}(60 \eta_\mathrm{l}^2EL^2+1)   \sum_{i=0}^{N} \mathbb{E}\left\|\delta_{i, k}^t-\delta_{i,0}^t\right\|^2+\frac{\eta_\mathrm{g} \eta_\mathrm{l} L^3 \rho^2}{2N} \sigma_l^2\right) \\ 
& \leq \frac{F\left(\tilde{w}^0\right)-f^*}{C E \eta_\mathrm{g} \eta_\mathrm{l} T}\\ 
& +\frac{1}{C}\left(10 E L^4 \eta_\mathrm{l}^2 \rho^2 \sigma_l^2+90 E^2 L^2 \eta_\mathrm{l}^2 \sigma_g^2+180 E^2 L^4 \eta_\mathrm{l}^2 \rho^2+\frac{L^2}{N}(60 \eta_\mathrm{l}^2EL^2+1)   \sum_{i=0}^{N} \mathbb{E}\left\|\delta_{i, k}^t-\delta_{i,0}^t\right\|^2+\frac{\eta_\mathrm{g} \eta_\mathrm{l} L^3 \rho^2}{2N} \sigma_l^2\right),\end{aligned}$$
if choosing $\eta_\mathrm{l}=\frac{1}{\sqrt{T} E L}$ and $\eta_\mathrm{g}=\sqrt{E N}$, under the intermediate results in Lemma~\ref{lem:2} of FedSAM we have:
$$\frac{1}{T} \sum_{t=1}^T \mathbb{E}\left[\left\|\nabla F\left(w^{t+1}\right)\right\|\right]\leq \left(\frac{10L(F\left(\tilde{w}^0\right)-F^*)}{C\sqrt{T E N}}+\frac{90L^2\rho^2\sigma_g^2}{CTE}+\frac{180 L^2\rho^2}{C T}+\frac{120L^2\rho^2}{CET^2}+\frac{2L^2\rho^2}{CT}+\frac{L^2\sigma_l^2 \rho^2}{C\sqrt{T E N}}\right)$$
Similarity, under the situation that $\eta_\mathrm{l}=\frac{1}{\sqrt{T} E L}$ and $\eta_\mathrm{g}=\sqrt{E N}$, with the intermediate results shown in Lemma~\ref{lem:2} of FedLESAM we have:
$$\frac{1}{T} \sum_{t=1}^T \mathbb{E}\left[\left\|\nabla F\left(w^{t+1}\right)\right\|\right]\leq \left(\frac{10L(F\left(\tilde{w}^0\right)-F^*)}{C\sqrt{T E N}}+\frac{90L^2\rho^2\sigma_g^2}{CTE}+\frac{180 L^2\rho^2}{C T}+\frac{L^2\sigma_l^2 \rho^2}{C\sqrt{T E N}}\right)$$

\end{proof}

\subsection{Proof of Theorem~2}
Here we provide the proof of Theorem~\ref{thm:naive}.

\begin{theorem}
    \label{thm:naive2}
    Assume local update is one step and follows Assumptions~\ref{asm:unit_variance}-~\ref{asm:smooth2}. Under full participation and $L_\mathrm{g}$-smoothness of $F$ with global and local learning rates $\eta_\mathrm{g}$ and $\eta_\mathrm{l}$, the estimation bias is bounded as
    \begin{equation}
    \small
    \|\frac{w^{t-1}-w^t}{\|w^{t-1}-w^t\|}-\frac{\nabla F(w^t)}{\|\nabla F(w)\|}\| \leq 3\sigma_\mathrm{l}'^2+3\sigma_\mathrm{g}'^2+3L_\mathrm{g}^2 \eta_\mathrm{g}^2\eta_\mathrm{l}^2.
    \nonumber
    \end{equation}
\end{theorem}

\begin{proof}
Under one step client updates and full participation, we have:
$$
w^t-w^{t-1}=\eta_\mathrm{g}\eta_\mathrm{l}\frac{1}{N}\sum_{i=1}^N \nabla F_i(w^{t-1},\xi_i).
$$Then the error bound can be defined as:
$$
error=\mathbb{E}\|\frac{\sum_{i=1}^N \nabla F_i(w^{t-1},\xi_i)}{\|\sum_{i=1}^N \nabla F_i(w^{t-1},\xi_i)\|}-\frac{\nabla F(w^t)}{\|\nabla F(w^t)\|}\|^2.
$$ Define $A=\frac{\sum_{i=1}^N \nabla F_i\left(w^{t-1}, b_i\right)}{\left\|\sum_{i=1}^N \nabla F_i\left(w^{t-1}, b_i\right)\right\|}-\frac{\sum_{i=1}^N \nabla F_i\left(w^{t-1}\right)}{\left\|\sum_{i=1}^N \nabla F_i\left(w^{t-1}\right)\right\|}$, $B=\frac{\sum_{i=1}^N \nabla F_i\left(w^{t-1}\right)}{\left\|\sum_{i=1}^N \nabla F_i\left(w^{t-1}\right)\right\|}-\frac{\nabla F\left(w^{t-1}\right)}{\| \nabla F\left(w^{t-1}\right)\|}$, and $C=\frac{\nabla F\left(w^{t-1}\right)}{\| \nabla F\left(w^{t-1}\right)\|}-\frac{\nabla F(w^t)}{\|\nabla F(w^t)\|}$. We have:
$$
error=\|A+B+C\|^2\leq 3\|A\|^2+3\|B\|^2+3\|C\|^2.
$$ 

Then we start to bound $\|C\|^2$. If the local and global learning rates are small and the gradient of global function $\nabla F\left(w^t, b\right)$ is small, based on the first order Hessian approximation, the expected gradient is
$$
\nabla F\left(w^t\right)=\nabla F\left(w^{t-1}+\eta_\mathrm{g}\eta_\mathrm{l}g^{t-1}\right)=\nabla F\left(w^{t-1}\right)+H \eta_\mathrm{g}\eta_\mathrm{l} g^{t-1}+O\left(\left\|\eta_\mathrm{g}\eta_\mathrm{l} g^{t-1}\right\|^2\right),
$$
where $H$ is the Hessian at $w^{t-1}$. Therefore, we have
$$
\mathbb{E}\left[\left\|\frac{\nabla F(w^{t-1})}{\|\nabla F(w^{t-1})\|}-\frac{\nabla F(w^t)}{\|\nabla F(w^t)\|}\right\|^2\right]\leq \phi^t,
$$
where $\phi^t$ is the square of the angle between the unit vector in the direction of $\nabla F(w^{t-1})$ and $\nabla F\left(w^t\right)$. The inequality is because (1) $\left\|\frac{\nabla F_i(\cdot)}{\left\|\nabla F_i(\cdot)\right\|}\right\|^2\leq1$, and thus we replace $\delta$ with a unit vector in the corresponding directions and obtain the upper bound, (2) the norm of difference between the unit vectors can be upper bounded by the square of the arc length on a unit circle. If the total learning rate $\eta_g\eta_l$ and the global model update $\nabla F\left(w^t\right)$ are small, $\phi^t$ will also be small. Based on the first order Taylor series, i.e., $\tan x=x+O\left(x^2\right)$, we have
$$
\begin{aligned} & \tan \phi^t=\frac{\left\|\nabla F(w^{t-1})-\nabla F\left(w^t\right)\right\|^2}{\left\|\nabla F\left(w^t\right)\right\|^2}+O\left((\phi^t)^2\right) \\ 
& =\frac{\left\|\nabla F(w^{t-1})-H \eta_\mathrm{g}\eta_\mathrm{l} g^{t-1}-O\left(\left\|\eta_\mathrm{g}\eta_\mathrm{l} g^{t-1}\right\|^2\right)-\nabla F(w^{t-1})\right\|^2}{\left\|\nabla F(w^{t-1})\right\|^2}+O\left((\phi^t)^2\right) \\ 
& \stackrel{\text { (a) }}{\leq} \eta_\mathrm{g}^2 \eta_\mathrm{l}^2L_\mathrm{g}^2,\end{aligned}
$$
where (a) is due to maximum eigenvalue of $H$ is bounded by $L_\mathrm{g}$ because $F$ function is $L_\mathrm{g}$-smooth.

Since $\|C\|^2$ is proved to be less than $L_\mathrm{g}^2\eta_\mathrm{l}^2\eta_\mathrm{g}^2$, and A and B are respectively bounded by the Assumptions~\ref{asm:unit_variance} and \ref{asm:unit_differ}, we have:
$$
\mathbb{E}\|\frac{w^{t-1}-w^t}{\|w^{t-1}-w^t\|}-\frac{\nabla F(w^t)}{\|\nabla F(w)\|}\| \leq 3\sigma_\mathrm{l}'^2+3\sigma_\mathrm{g}'^2+3L_\mathrm{g}^2 \eta_\mathrm{g}^2\eta_\mathrm{l}^2.
$$ Here we complete the proof.
\end{proof}

\begin{figure*}[ht!]
    \centering
    \subfigure[CIFAR10, 100 Clients, $\beta=0.6$]{
    \centering
    \label{fig:cifar10p100b06}
    \includegraphics[width=0.3\textwidth]{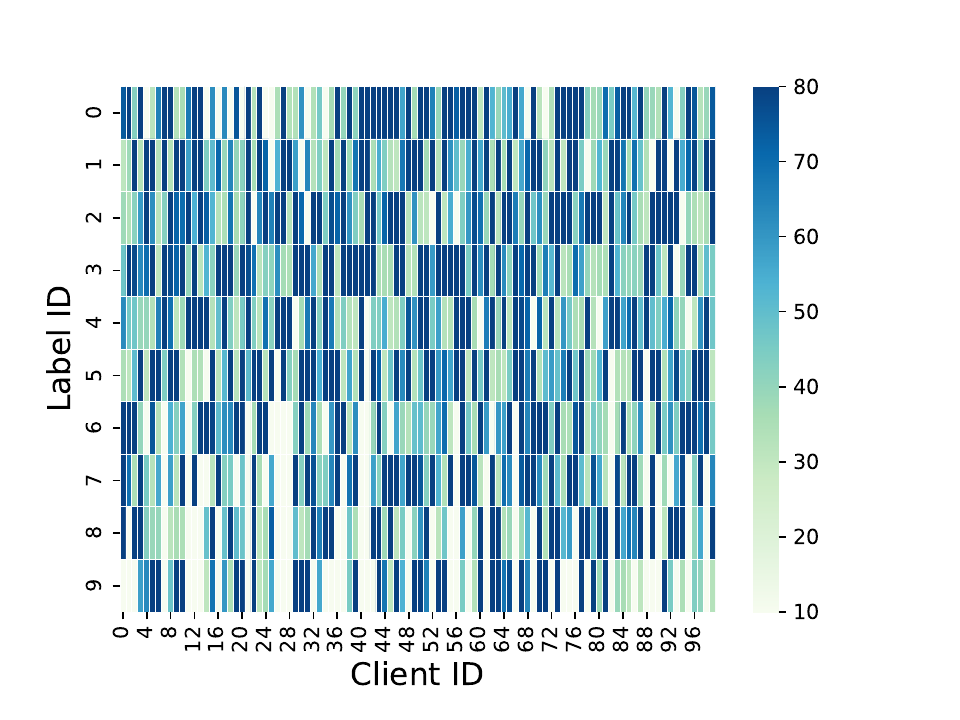}}
    %\hspace{2mm}
    \subfigure[CIFAR10, 100 Clients, $\beta=0.1$]{
    \centering
    \label{fig:cifar10p100b01}
    \includegraphics[width=0.3\textwidth]{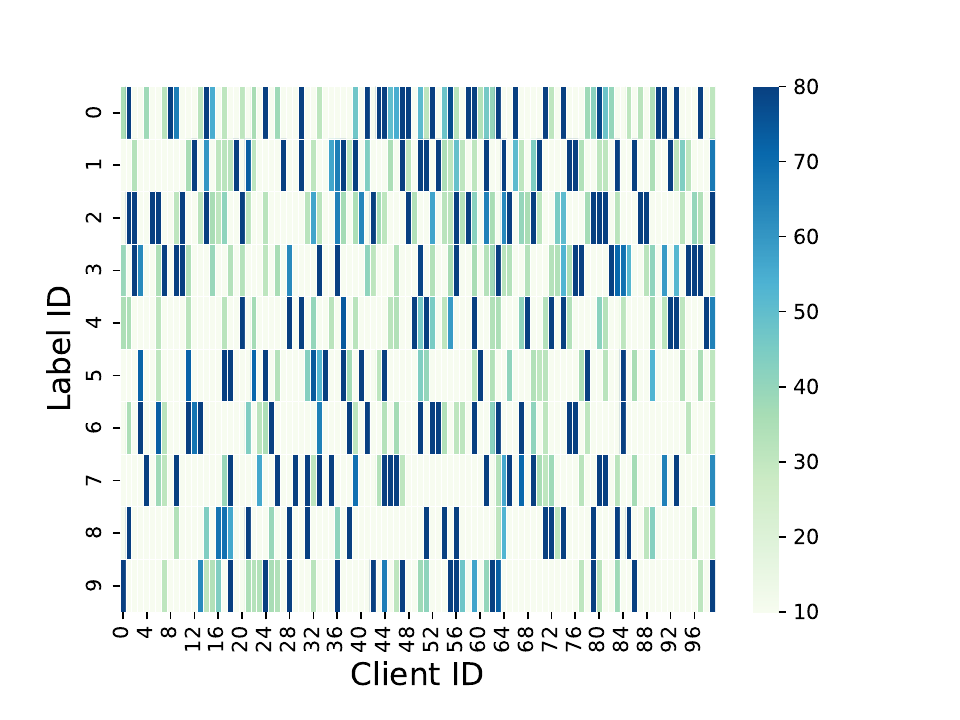}}
    \subfigure[CIFAR10, 200 Clients, $\beta=0.6$]{
    \centering
    \label{fig:cifar10p200b06}
    \includegraphics[width=0.3\textwidth]{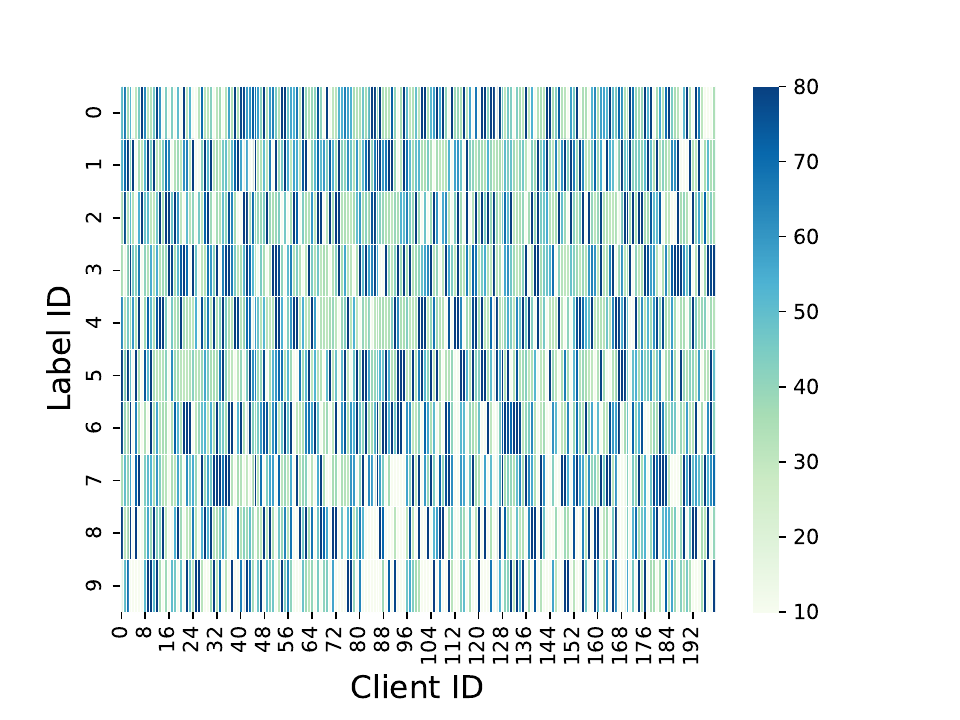}}
    \subfigure[CIFAR10, 200 Clients, $\beta=0.1$]{
    \centering
    \label{fig:cifar10p200b01}
    \includegraphics[width=0.3\textwidth]{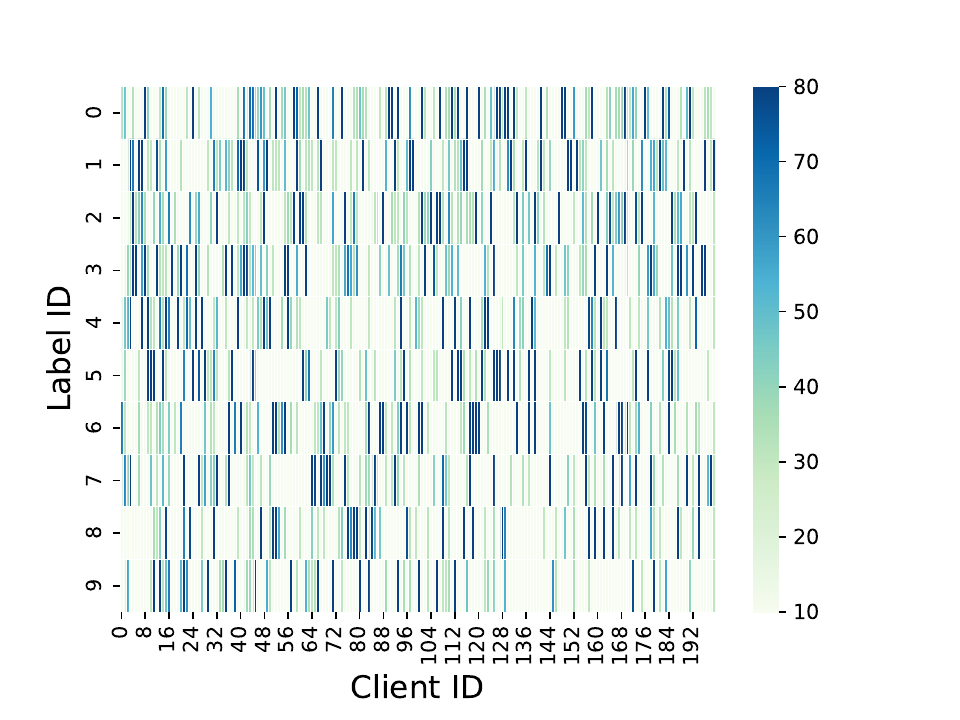}}
    \centering
    \subfigure[CIFAR100, 100 Clients, $\beta=0.6$]{
    \centering
    \label{fig:cifar100p100b06}
    \includegraphics[width=0.3\textwidth]{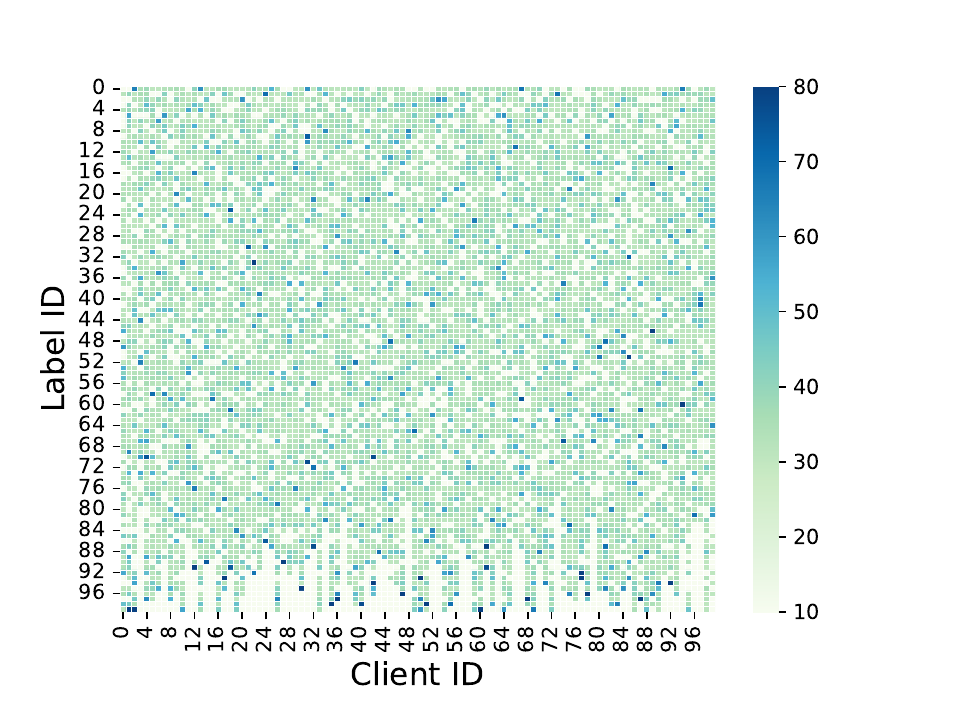}}
    %\hspace{2mm}
    \subfigure[CIFAR100, 100 Clients, $\beta=0.1$]{
    \centering
    \label{fig:cifar100p100b01}
    \includegraphics[width=0.3\textwidth]{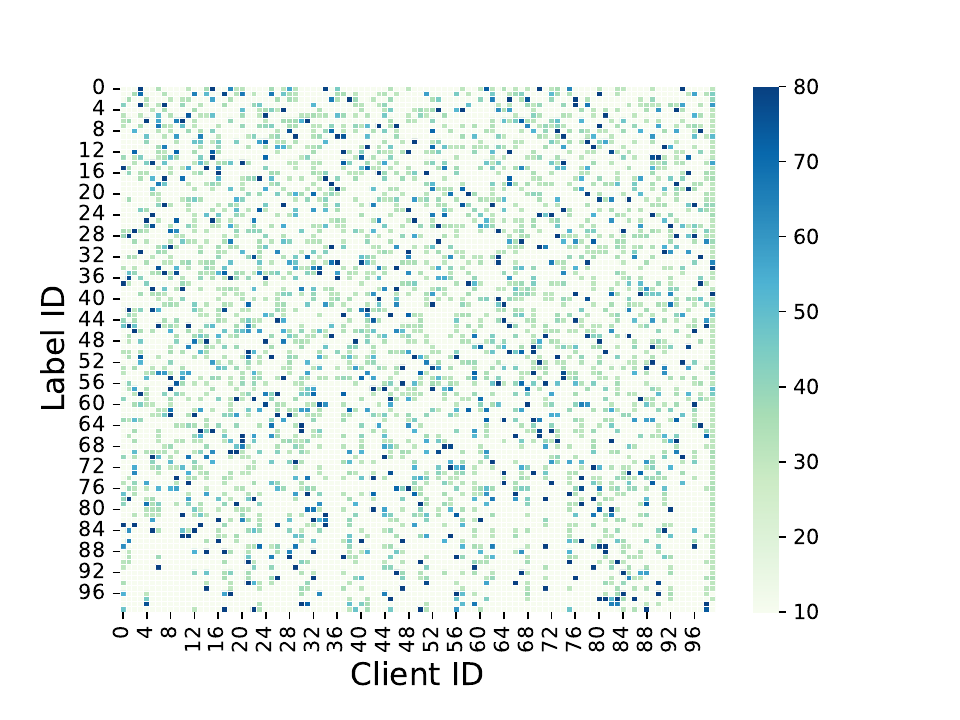}}
    \subfigure[CIFAR100, 200 Clients, $\beta=0.6$]{
    \centering
    \label{fig:cifar100p200b06}
    \includegraphics[width=0.3\textwidth]{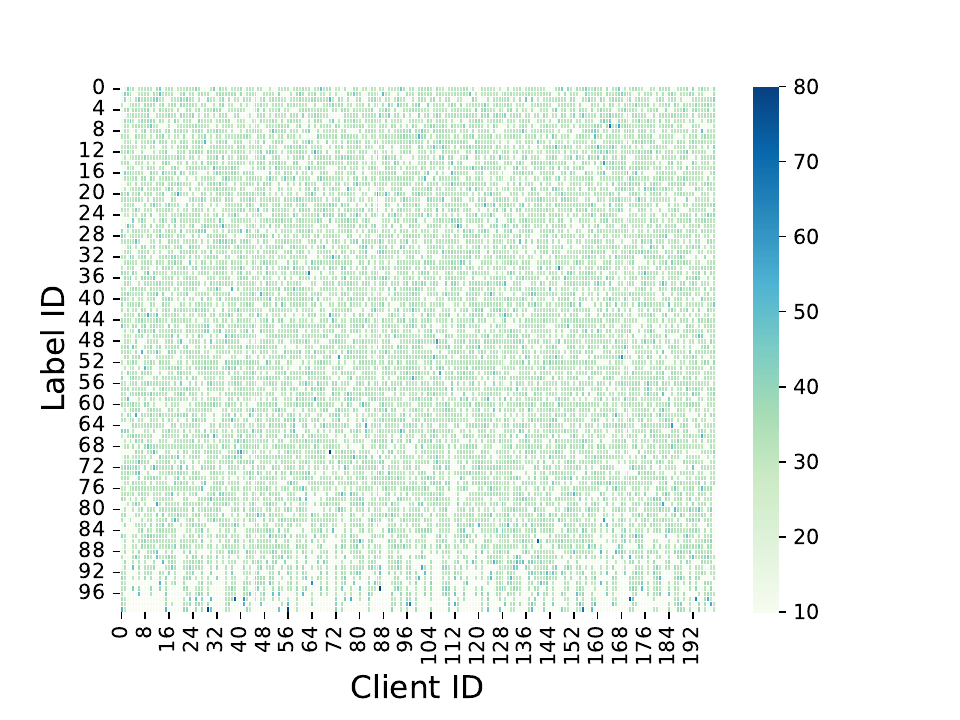}}
    \subfigure[CIFAR100, 200 Clients, $\beta=0.1$]{
    \centering
    \label{fig:cifar100p200b01}
    \includegraphics[width=0.3\textwidth]{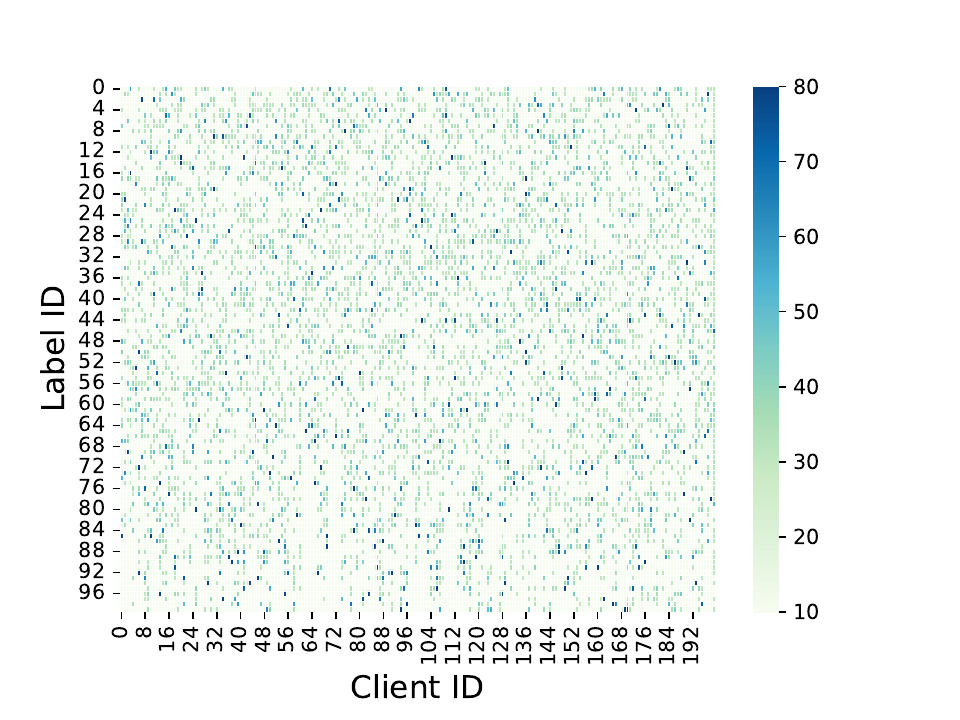}}
    \vspace{-8pt}
    \caption{Heatmap of data distribution of CIFAR10 and CIFAR100 under Dirichelet distribution with coefficients $\beta$ of $0.6$ and $0.1$. The two datasets are divided into 100 and 200 clients. }\label{fig:heatmap}
\end{figure*}

\section{Implementation of the Experiments}\label{app:data}
This section presents some details about benchmark datasets, data split strategies and backbone models used in the experiments.

\subsection{Datasets} 
\begin{table}[h]
\centering
\caption{\rm{A summary of CIFAR10/100, OfficeHome, and DomainNet datasets, including number of total images, number of classes, number of domains and the size of the images in the datasets.}}
\vspace{5pt}
    \scalebox{1.00}{
    \begin{tabular}{l|cccc}
        \toprule[2pt]
        Dataset&  Total Images & Class& Domain&Image Size\\
        \midrule
        CIFAR10&60,000 &10&-&3$~\times~$32$~\times~$32 \\
        CIFAR100&60,000 &100&-&3$~\times~$32$~\times~$32 \\
        OfficeHome&15,588&65&4&3$~\times~$224$~\times~$224\\
        DomainNet&586,575&345&6&3$~\times~$224$~\times~$224 \\
        \bottomrule[2pt]
    \end{tabular}}
    \label{tb:dataset}
\end{table}
CIFAR-10/100~\cite{cifar10}, OfficeHome~\cite{officehome} and DomainNet~\cite{domainnet} are all popular benchmark datasets in the field of federated learning. Data samples in CIFAR10 and CIFAR100 are colorful images of different categories with the resolution of $32~\times~32$. 
There are 10 classes and each class has 6,000 images in CIFAR10. 
For CIFAR100, there are 100 classes and each class has 600 images. Fore OfficeHome, there are 65 classes and 4 domains with 15,588 images~(resolution of $224~\times~224$). 
DomainNet is a large dataset, which has 345 classes and 6 domains with 586,575 images~(resolution of $224~\times~224$). 
As shown in the Table~\ref{tb:dataset}, we summarize CIFAR10, CIFAR100, OfficeHome, and DomainNet from the views of number of total images, number of classes, number of domains and the size of the images in the datasets.

\subsection{Splits}
For CIFAR10 and CIFAR100, we follow the settings in \citet{split}, \citet{fedgamma}, and \citet{fedsamicml} and use Dirichlet distribution and Pathological split strategies to simulate the situations of Non-IID. As shown in the Figure~\ref{fig:heatmap}, we provide the heatmap of data distribution among clients of CIFAR10 and CIFAR100 under Dirichlet distribution with coefficients of 0.6 and 0.1. The two datasets are divided into 100 and 200 clients. It can be seen that, the split can generate practical and complicated data distribution. 
For OfficeHome and DomainNet, we adopt the standard leave-one-domain-out split strategy that selects one domain for test and all other domains for federated training.

\subsection{Model}
Resnet18 backbone is commonly used in many experiments on CIFAR10 and CIFAR100 datasets~\cite{fedsmoo,fedspeed,sun1,sun2,fedgela,hong1,hong2,hong3}, here we also use it as the backbone followed with a classification head. Following the advice of \citet{advice} and keeping the same setting with \citet{fedsmoo,fedgamma} to avoid the non-differentiable parameters, we replace
the Batch Normalization with the Group Normalization~\cite{groupnorm}. To validate the performance of algorithms on different models, for DomainNet and OfficeHome, we adopt the pre-trained ViT-B/32~\cite{vit} as the backbone. 
\section{Variants}\label{app:variants}
In this section, we show the process of an optimization method in federated learning named Scaffold~\cite{scaffold}, and introduce the procedures of our two variants based on the frameworks of FedDyn~\cite{fedyn} and Scaffold, named FedLESAM-D and FedLESAM-S respectively. Client loss surfaces may not align with the global loss surface, meaning that minimizing local sharpness in FedSAM and MoFedSAM might not effectively reduce global sharpness. Effective variance reduction through Scaffold~($w_{i,k}^t$ aligns $w_{g,k}^t$ during local training) enables FedGAMMA to reduce both the training loss and the upper bound of global sharpness. In FedLESAM-S and FedLESAM-D, effective variance reduction combined with an accurate estimate of global perturbation leads to directly minimizing both training loss and global sharpness. With successful estimation and variance reduction, the key difference between FedGAMMA and FedLESAM-S is that FedGAMMA minimizes the upper bound of global sharpness, whereas FedLESAM-S directly minimizes the global sharpness. 

\begin{algorithm}[t!]
    \caption{Scaffold}
    \label{alg:scaffold}
    \textbf{Input}:$(K, \rho, w^0, E, T, \eta_\mathrm{l}, \eta_\mathrm{g}, \forall i~C_i=0, C=0)$
    \begin{algorithmic}
    \FOR{$t = 0,1,\dots,T-1$}
        \FOR{sampled $n$ active client $i = 1,2,\dots,n$}
            \STATE receive $w^t$, $w^t_{i,0}\leftarrow w^t$
            \FOR{$k=0,1,...,E-1$} 
                \STATE sample a batch of data $b_{i,k}^t$
                \STATE $w_{i,k+1}^{t} \leftarrow w_{i,k}^t - \eta_\mathrm{l} \nabla\mathcal{L}(w_{i,k}^t; b_{i,k}^t)+\eta_\mathrm{l}(C-C_i)$
            \ENDFOR
                \STATE $C_i=C_i-C+\frac{1}{\eta_\mathrm{l}E}(w^t-w^t_{i,E})$
                \STATE submit $C_i$ and $w^t_{i,E}$.
        \ENDFOR
        \STATE $w^{t+1} \leftarrow w^t-\eta_\mathrm{g}\sum_{i=1}^{K}{w^t-w^t_{i,E}}$.
        \STATE $C=C+\frac{1}{K}C_i$
    \ENDFOR
    \end{algorithmic}
{\textbf{Output}:$w^T.$}
\end{algorithm} 
\subsection{Scaffold and Comparison}
\citet{scaffold} proposed Scaffold to reduce the client drift by introducing variance reduction. Scaffold estimates the update direction for the server model and the update direction for each client, denoted as $C$ and  $C_i$, respectively. The difference $(C-C_i)$ is used to correct the local update. As shown in the Algorithm~\ref{alg:scaffold}, we provide the procedure of Scaffold. It can be seen that, under full participation case, $C$ is equal to $w^{t-1}-w^{t}$, which is the same in our algorithm to estimate global gradient. However, in Scaffold, it is used as global gradient for correcting local updates, while in our FedLESAM, it is used as global gradient to estimate global perturbation. 

\subsection{FedLESAM-S}
\begin{algorithm}[t!]
    \caption{FedLESAM-S}
    \label{alg:fedlesam-s}
    \textbf{Input}:$(K, \rho, w^0, E, T, \eta_\mathrm{l}, \eta_\mathrm{g}, \forall i~ w_i^\mathrm{old}=0, \forall i~C_i=0, C=0)$
    \begin{algorithmic}
    \FOR{$t = 0,1,\dots,T-1$}
        \FOR{sampled $n$ active client $i = 1,2,\dots,n$}
            \STATE receive $w^t$, $w^t_{i,0}\leftarrow w^t$
            \FOR{$k=0,1,...,E-1$} 
                \STATE sample a batch of data $b_{i,k}^t$
                \STATE \colorbox{gray!20}{$\rhd$ perturbation stage}
                \STATE $\delta^t_{i,k}=\rho \frac{w_i^\mathrm{old}-w^t}{\|w_i^\mathrm{old}-w^t\|}$ 
                \STATE $w_{i,k+1}^{t} \leftarrow w_{i,k}^t - \eta_\mathrm{l} \nabla\mathcal{L}(w_{i,k}^t+\delta^t_{i,k} ; b_{i,k}^t)+\eta_\mathrm{l}(C-C_i)$
            \ENDFOR
                \STATE $C_i=C_i-C+\frac{1}{\eta_\mathrm{l}E}(w^t-w^t_{i,E})$
                \STATE store $w_i^\mathrm{old}=w^t$
                \STATE submit $w^{t}_{i,E}$ and $C_i$.
        \ENDFOR
        \STATE $w^{t+1} \leftarrow w^t-\eta_\mathrm{g}\sum_{i=1}^{K}{w^t-w^t_{i,E}}$.
        \STATE $C=C+\frac{1}{K}C_i$
    \ENDFOR
    \end{algorithmic}
{\textbf{Output}:$w^T.$}
\end{algorithm} 
Here we introduce our variant FedLESAM-S based on the framework Scaffold. As illustrated in the Algorithm~\ref{alg:fedlesam-s}, FedLESAM-S locally estimates the global perturbation and incorporates variance reduction of Scaffold into the local training. Other procedures like communication and local update correction are the same with Scaffold.

\subsection{FedLESAM-D}
As illustrated in the Algorithm~\ref{alg:fedlesam-d}, we provide the procedure of our FedLESAM-D based on the framework of FedDyn. We incorporate the regularizer in FedDyn to correct local updates. In the experiments, we found it not stable during the federated training and the overfitting problem is easy to happen, as well as FedDyn and FedSMOO.

\begin{algorithm}[t!]
    \caption{FedLESAM-D}
    \label{alg:fedlesam-d}
    \textbf{Input}:$(K, \rho, w^0, E, T, \eta_\mathrm{l}, \beta, \eta_\mathrm{g}, \forall i~ w_i^\mathrm{old}=0, \forall i~\lambda_i=0, \lambda=0)$
    \begin{algorithmic}
    \FOR{$t = 0,1,\dots,T-1$}
        \FOR{sampled $n$ active client $i = 1,2,\dots,n$}
            \STATE receive $w^t$, $w^t_{i,0}\leftarrow w^t$
            \FOR{$k=0,1,...,E-1$} 
                \STATE sample a batch of data $b_{i,k}^t$
                \STATE \colorbox{gray!20}{$\rhd$ perturbation stage}
                \STATE $\delta^t_{i,k}=\rho \frac{w_i^\mathrm{old}-w^t}{\|w_i^\mathrm{old}-w^t\|}$ 
                % \STATE $\hat{w}_{i,k}^t=w_{i,k}^t+\delta^t_{i,k}$
                \STATE $w_{i,k+1}^{t} \leftarrow w_{i,k}^t - \eta_\mathrm{l} \nabla\mathcal{L}(w_{i,k}^t+\delta^t_{i,k} ; b_{i,k}^t)-\eta_\mathrm{l}(\lambda_i+\frac{1}{\beta}\left(w_{i, k}^t-w^t\right))$
            \ENDFOR
                \STATE store $w_i^\mathrm{old}=w^t$
                \STATE submit $w^{t}_{i,E}.$
                \STATE $\lambda_i=\lambda_i-\frac{1}{\beta}\left(w_{i, k}^t-w^t\right)$
        \ENDFOR
        \STATE $w^{t+1} \leftarrow w^t-\eta_\mathrm{g}\sum_{i=1}^{K}{w^t-w^t_{i,E}}$.
        \STATE $\lambda^{t+1}=\lambda^t-\frac{1}{\beta K} \sum_{i =1}^K\left(w_{i,K}^t-w^t\right)$
    \ENDFOR
    \end{algorithmic}
{\textbf{Output}:$w^T.$}
\end{algorithm}

\begin{figure*}[ht!]
    \centering
    \subfigure[FedAvg]{
    \centering
    \label{fig:surface_12}
    \includegraphics[width=0.31\textwidth]{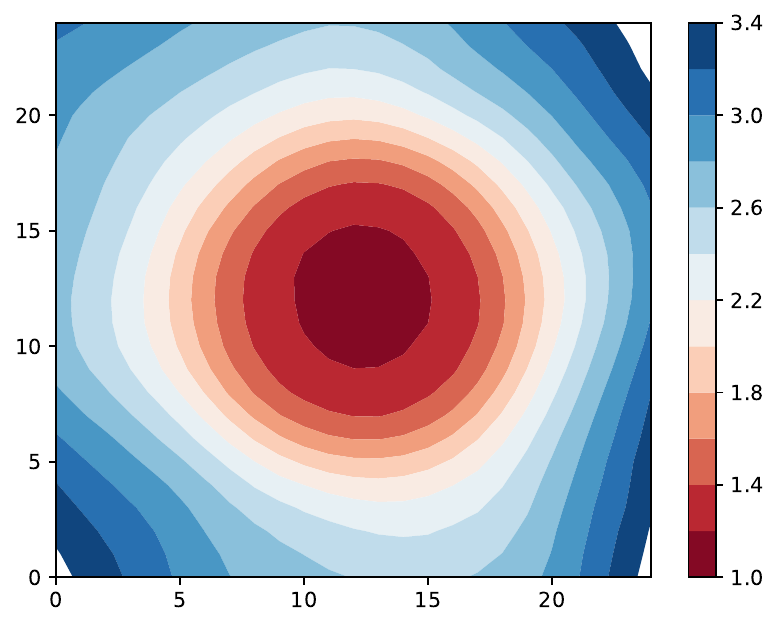}}
    %\hspace{2mm}
    \subfigure[FedSAM]{
    \centering
    \label{fig:surface_22}
    \includegraphics[width=0.31\textwidth]{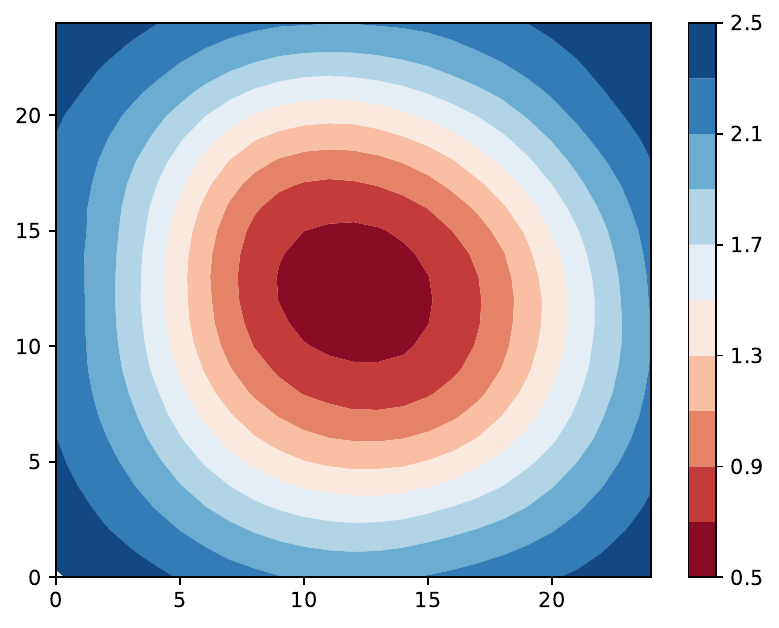}}
    \subfigure[FedGAMMA]{
    \centering
    \label{fig:surface_32}
    \includegraphics[width=0.31\textwidth]{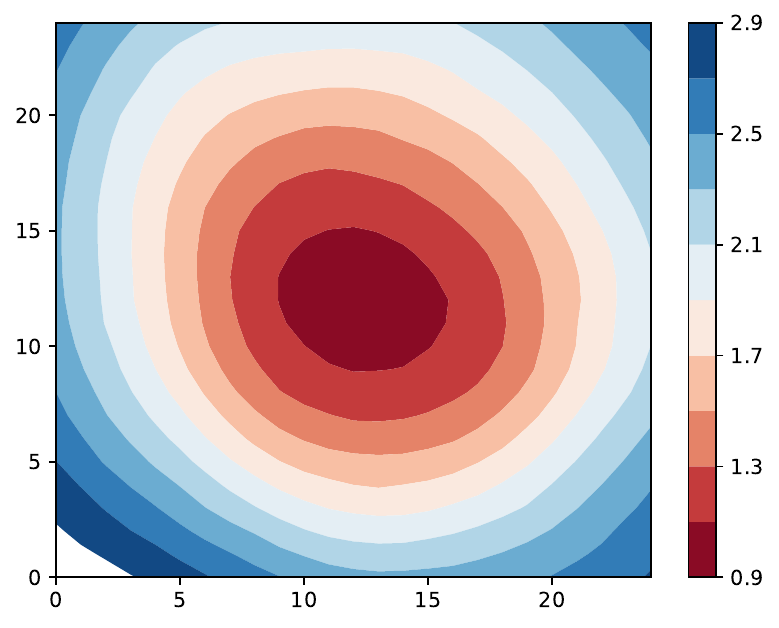}}
    \subfigure[MoFedSAM]{
    \centering
    \label{fig:surface_42}
    \includegraphics[width=0.31\textwidth]{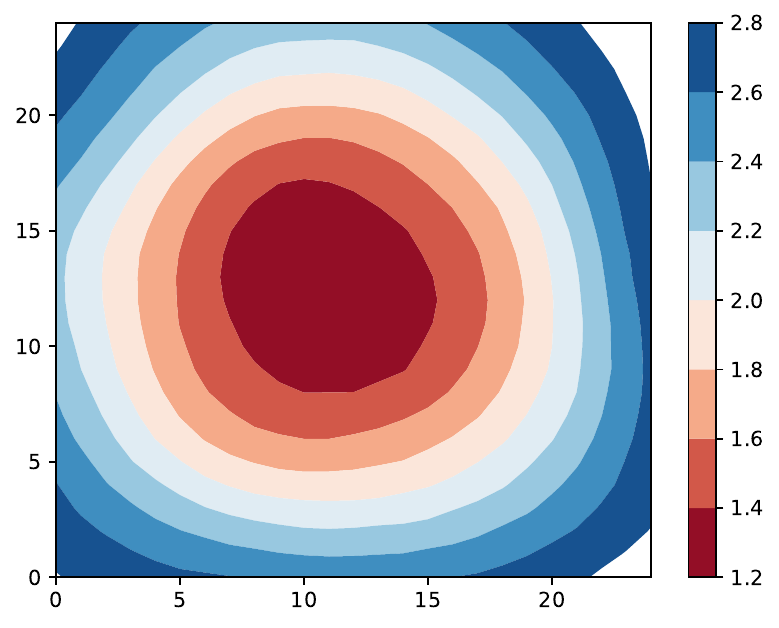}}
    \centering
    \subfigure[FedSMOO]{
    \centering
    \label{fig:surface_52}
    \includegraphics[width=0.31\textwidth]{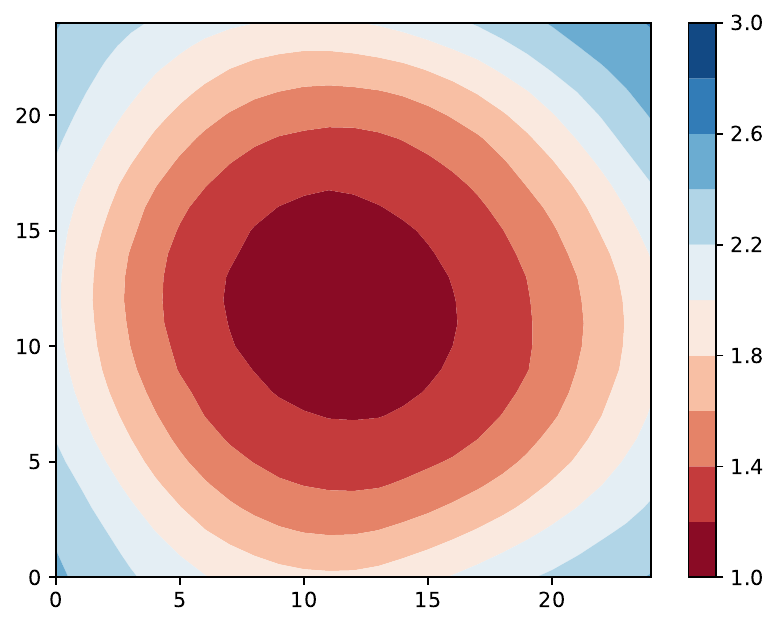}}
    %\hspace{2mm}
    \subfigure[FedLESAM-D]{
    \centering
    \label{fig:surface_62}
    \includegraphics[width=0.31\textwidth]{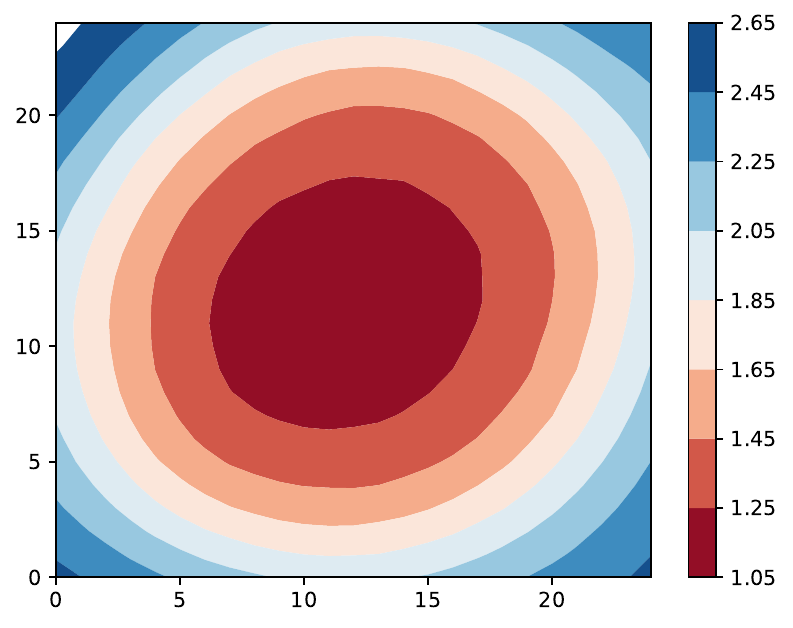}}
    \vspace{-8pt}
    \caption{Visualization of the global loss surface on CIFAR10 under Pathological distribution with coefficient 3 of FedAvg, FedSAM, FedGAMMA, MoFedSAM, FedSMOO and our FedLESAM-D. We divide the dataset into 100 clients and in each round 10\% clients are active.}\label{fig:surface2}
    \vspace{-8pt}
\end{figure*} 

\begin{table}[h]
\centering
\caption{\rm{Stored memory, backpropagation in each local step and communication at each round compared to FedAvg for SAM-based federated methods.}}
\vspace{5pt}
    \setlength{\tabcolsep}{1.5pt}
    \scalebox{1.00}{
    \begin{tabular}{l|cccccccc}
        \toprule[2pt]
        &  FedAvg&	FedSAM&	FedLESAM&	MoFedSAM&	FedSMOO&	FedLESAM-D&	FedGAMMA&	FedLESAM-S\\
        \midrule
        Stored memory&1$~\times~$&	1$~\times~$&	2$~\times~$&	2$~\times~$&	4$~\times~$&	4$~\times~$&	3$~\times~$&	4$~\times~$ \\
        Communication&1$~\times~$&	1$~\times~$&	1$~\times~$&	2$~\times~$&	2$~\times~$&	2$~\times~$&	2$~\times~$&	2$~\times~$ \\
        Backpropagation&1$~\times~$&	2$~\times~$&	1$~\times~$&	2$~\times~$&	2$~\times~$&	1$~\times~$&	2$~\times~$&	1$~\times~$\\
        \bottomrule[2pt]
    \end{tabular}}
    \label{tb:cc}
    \vspace{-10pt}
\end{table}

\begin{table}[h]
\centering
\caption{\rm{Wall clock time~(times including training, loading and evaluation) on one GeForce RTX 3090 between two communications on CIFAR10 under dirichlet distribution with coefficient $\beta=0.6,0.1$ and 100 clients. The active ratio is 10\%.}}
\vspace{5pt}
    \setlength{\tabcolsep}{1.5pt}
    \scalebox{1.00}{
    \begin{tabular}{l|cccccccc}
        \toprule[2pt]
        &  FedAvg&	FedSAM&	MoFedSAM&	FedGAMMA&	FedSMOO&	FedLESAM&	FedLESAM-S&	FedLESAM-D\\
        \midrule
        wall clock time	&20.34s	&25.71s	&28.73s	&29.88s	&29.67s	&20.99s	&25.81s	&25.70s\\
        \bottomrule[2pt]
    \end{tabular}}
    \label{tb:wt}
    \vspace{-10pt}
\end{table}

\begin{table}[t!]
\centering
\small
\renewcommand\arraystretch{0.95}
\caption{\rm{Total communication rounds, computational time (minutes) and communication costs (gigabytes of parameters) to achieve 68\% and 74\% test accuracy under Pathological split with coefficient of 3, 100 clients and 10\% active ratio on CIFAR10 of FedAvg, SAM-based methods, and our two variants.}}
\vspace{4pt}
    \setlength{\tabcolsep}{1pt}
    \scalebox{0.95}{
    \begin{tabular}{l|cc|cc|cc}
        \toprule[2pt]
        Method& \multicolumn{2}{c|}{Commu. Round}  & \multicolumn{2}{c|}{Commu. Cost}& \multicolumn{2}{c}{Compu. Time}\\
        \cmidrule{1-7} \#Target Acc.& 68\% & 74\% & 68\% & 74\%& 68\% & 74\%\\
        \midrule
        FedAvg&246 (1x)&723 (1x) &53 (1x)&156 (1x)&27 (1x)&80 (1x)\\
        FedSAM&253&604&1.03x&0.84x&1.97x&1.60x\\
        MoFedSAM&116&298&0.94x&0.82x&1.08x&0.95x\\
        FedGAMMA&176&422&1.43x&1.17x&1.57x&1.28x\\
        FedSMOO&144&194&1.17x&0.54&1.32x&0.61x\\
        \midrule
        FedLESAM-S&159&332&1.29x&0.92x&0.92x&0.65x\\
        FedLESAM-D&141&182&1.15x&0.50x&1.03x&0.45x\\
        \bottomrule[2pt]
    \end{tabular}}
    \label{tab:commu2}
    \vspace{-10pt}
\end{table}

\section{More Information in the Experiments} \label{app:implement}
\subsection{Hyper-prarameter choosing}
For a fair comparison on CIFAR10 and CIFAR100, we follow all the settings in FedGAMMA~\cite{fedgamma} and FedSMOO~\cite{fedsmoo}. Backbone is ResNet-18~\cite{resnet} with the Group Normalization~\cite{groupnorm} and SGD, total rounds $T=800$, initial local learning rate $\eta_\mathrm{l}=0.1$, global learning rate $\eta_\mathrm{g}=1$ except for FedAdam which adopts $0.1$, perturbation magnitude $\rho$ equals to $0.1$ for FedGAMMA, FedSMOO and the corresponding variants of our FedLESAM-S and FedLESAM-D, $\rho=0.01$ for FedSAM and our original FedLESAM, weight decay equals to 1e-3, and learning rate decreases by 0.998× exponentially except for FedDyn, FedSMOO and FedLESAM-D which adopt 0.9995× for the proxy term. On the CIFAR10, batchsize is 50, and the local epochs is 5. On the CIFAR100, batchsize equals to 20, and the local epochs equal to 2. For OfficeHome and DomainNet, we use the pre-trained model ViT-B/32~\cite{vit} as the backbone to verify the robustness of algorithms on different models. The optimizer is SGD with local learning rate 0.001 and global learning rate 1 except for FedAdam which adopts 0.1. For all methods, $\rho$ is tuned from $\{0.05, 0.01, 0.005, 0.001, 0.0005\}$, local epochs is 5, batchsize is 32 and total communication rounds equal to 400.

\subsection{More Loss Surface Visualization}
Here we show more visualizations of the global loss surfaces. As shown in Figure~\ref{fig:surface2}, we conduct experiments on CIFAR10 under Pathological splits with a coefficient of 3 and visualize the global loss surface of FedAvg, FedSAM, FedGAMMA, MoFedSAM, FedSMOO and our FedLESAM-D. Among all algorithms, our FedSMOO and FedLESAM-D achieve the much flatter loss landscape.
% \subsection{More Computation and Communication Results}

\subsection{Communication and Computation}
As shown in Table~\ref{tb:cc}, we provide the communication at each round, memory stored in clients and backpropagation performed in each local step compared to FedAvg. Our variants' storing and communication are comparable to FedSMOO and we reduce much computation. Compared to FedAvg and FedSAM, FedLESAM and MoFedSAM doubles the memory. Compared to FedSMOO, our FedLESAM-D maintains the same memory level. Compared to FedGAMMA, our FedLESAM-S requires an additional 33\% memory. As shown in Table~\ref{tb:wt}, we provide the wall-clock time comparison in the average time between two communications. It can be seen that our method greatly saves computational times. To further demonstrate efficiency, we provide results of an additional case in the Table~\ref{tab:commu2}. Our method exhibits competting communication efficiency and significantly reduces computation.

% \subsection{More Ablation Results of $\rho$}

\end{document}